\documentclass[12pt]{article}
\usepackage{amsmath,amssymb,amsopn,amsfonts,pdfpages,algorithmic,algorithm,dsfont,amsthm}
\usepackage[top=1in,bottom=1in, left=1in, right=1in]{geometry}
\usepackage{tabularx,dsfont}
\usepackage{bbm}
\usepackage{mathtools}
\usepackage{enumitem}
\usepackage[english]{babel}
\usepackage{xcolor}
\usepackage{relsize}
\usepackage{graphics}
\usepackage{caption}
\usepackage{url}
\usepackage{algorithmic}
\usepackage{graphicx}
\usepackage{graphics}
\usepackage{subfigure}
\usepackage[colorlinks=true,linkcolor=cyan,citecolor=blue]{hyperref}

\usepackage{stackengine}

\usepackage{verbatim}
\usepackage{bm}

\DeclareMathOperator{\EX}{\mathbb{E}}

\newtheorem{theorem}{Theorem}

\newtheorem{definition}{Definition}[section]
\newtheorem{defn}{Definition}[section]

\newtheorem{lemma}{Lemma}

\newcommand{\p}{\mathbb{P}}
\newcommand{\R}{\mathbb{R}}

\newcommand{\bJm}{\mathbb{J}_{m-1}}
\newcommand{\bJmm}{\mathbb{J}_{m-1,m}}

\newcommand{\sql}{\sum_{q}\sum_{\ell:\ell>q}}

\newcommand{\bde}{\begin{defn}}
\newcommand{\ede}{\end{defn}}

\newcommand{\bX}{\mathbf{X}}

\newcommand{\fM}{\mathfrak{M}}

\newcommand{\fR}{\mathfrak{R}}
\newcommand{\fr}{\mathfrak{r}}
\newcommand{\bD}{\mathbf{D}}
\newcommand{\bL}{\mathbf{L}}
\newcommand{\bR}{\mathbf{R}}
\newcommand{\bJ}{\mathbf{J}}
\newcommand{\bA}{\mathbf{A}}
\newcommand{\bP}{\mathbf{P}}
\newcommand{\bhX}{\hat{\mathbf{X}}}
\newcommand{\tcA}{\tilde{\mathcal{A}}}
\newcommand{\cA}{\mathcal{A}}

\makeatletter
\def\thm@space@setup{%
 \thm@preskip=\parskip \thm@postskip=0pt\label{key}
}
\makeatother
\makeatletter
\let \@fnsymbol\@arabic
\makeatother

\title{Optimizing the Induced Correlation in Omnibus Joint Graph Embeddings }
\author{%
Konstantinos Pantazis\thanks{Department of Applied Mathematics and Statistics, Johns Hopkins University; Note: Konstantinos currently works for Deus Ex Machina.  His contribution to this paper was completed during his doctoral work at University or Maryland and postdoctoral fellowship at Johns Hopkins University},
Michael Trosset\thanks{Department of Statistics, Indiana University},
William N. Frost\thanks{Cell Biology and Anatomy, and Center for Brain Function and Repair, Chicago Medical School, Rosalind Franklin University of Medicine and Science},\\ 
Carey E. Priebe\thanks{Department of Applied Mathematics and Statistics and Center for Imaging Sciences, Johns Hopkins University},
Vince Lyzinski\thanks{Department of Mathematics,  University or Maryland}%
}

\date{\today}

\begin{document}

\maketitle
\begin{abstract}
Theoretical and empirical evidence suggests that joint graph embedding algorithms induce correlation across the networks in the embedding space.
In the Omnibus joint graph embedding framework, previous results explicitly delineated the dual effects of the algorithm-induced and model-inherent correlations on the correlation across the embedded networks.
Accounting for and mitigating the algorithm-induced correlation is key to subsequent inference, as sub-optimal Omnibus matrix constructions have been demonstrated to lead to loss in inference fidelity.
This work presents the first efforts to automate the Omnibus construction in order to address two key questions in this joint embedding framework:
the \textit{correlation--to--OMNI} problem and the \textit{flat correlation} problem. 
In the flat correlation problem, we seek to understand the minimum  algorithm-induced flat correlation (i.e., the same across all graph pairs) produced by a generalized Omnibus embedding.
Working in a subspace of the fully general Omnibus matrices, we prove both a lower bound for this flat correlation and that the classical Omnibus construction induces the maximal flat correlation. 
In the correlation--to--OMNI problem, we present an algorithm---named \texttt{corr2Omni}---that, from a given matrix of estimated pairwise graph correlations, estimates the matrix of generalized Omnibus weights that induces optimal correlation in the embedding space.
Moreover, in both simulated and real data settings, we demonstrate the increased effectiveness of our \texttt{corr2Omni} algorithm versus the classical Omnibus construction.
\end{abstract}

\section{Introduction}
\label{sec:introch4}

Inference across multiple networks is an important area of statistical network analysis, and there is a large literature devoted to developing methods for classical multi-sample inference tasks in the setting of network data.
One important tool in multiple network inference is network embeddings, which, for the vertices across the collection of networks, provide suitable low-dimensional representations more amenable to classical inferential methods.
Embedding multiple networks is often achieved by jointly embedding the collection of networks (see, e.g., \cite{levin2017central,wang2019joint,nielsen2018multiple,arroyo2021inference,pantazis2022importance,macdonald2022latent}) or computing separate embeddings and then computing pairwise statistics to account for any nonidentifiability in the embedding space (see, e.g., \cite{athreya2022discovering,chen2024euclidean,cape2021spectral}).
Our focus in this paper is joint network embeddings, which have proven useful across myriad inference tasks ranging from time-series analysis \cite{chen2020multiple,pantazis2022importance} to two-sample network testing \cite{levin2017central,zheng2022limit} to joint graph alignment \cite{nenning2020joint} to name a few.

In many cases, the price that joint embedding methods pay for bypassing the need to compute pairwise statistics (e.g., Procrustean rotations as in \cite{cape2021spectral} or distance metrics as in \cite{athreya2022discovering})
is a level of induced correlation across the networks in the embedding space.
The work in \cite{pantazis2022importance} was the first to precisely characterize this induced correlation in the setting of the Omnibus embedding method of \cite{levin2017central}.
Working in the (Joint) Random Dot Product Graph model of \cite{young2007random}, Theorem 14 in \cite{pantazis2022importance} was able to precisely delineate the level of induced correlation across embeddings in terms of the method induced correlation (induced here by the Omnibus construction) and the inherent model correlation across networks.
Moreover, careful constructions of the Omnibus matrix were shown to induce  correlation levels that were more amenable to subsequent inference goals.
The specific Omnibus constructions in \cite{pantazis2022importance} were ad hoc, and this motivates our present work: namely, we seek to answer two related questions about the Omnibus construction, stated informally below.
\begin{itemize}
\item[i.] If we seek to have equal correlation induced between any graph pair, how small can this induced correlation be?
\item[ii.] If the inherent correlation matrix between the graph collection is $\mathbf{R}$, can we construct an Omnibus embedding that induces correlation as  close to $\mathbf{R}$ as possible?
\end{itemize}
Before formally stating the above questions, we must first present the necessary background on the Omnibus embedding framework.

\subsection{Omnibus Embeddings}
\label{sec:omni}

The Omnibus embedding framework introduced in \cite{levin2017central} (written in shorthand as the OMNI embedding) is a methodology for simultaneously embedding a collection of networks 
$$\{\bA^{(1)},\bA^{(2)},\cdots,\bA^{(m)}\}$$ into a low-dimensional Euclidean space.
The methodology consists of two steps. First, construct the OMNI matrix, and then embed it using Adjacency Spectral Embedding.
We will first present the more general OMNI construction of \cite{pantazis2022importance}, as the classical OMNI construction of \cite{levin2017central} can be realized as a special case.
\begin{defn}[Generalized Omnibus Matrix]
\label{def:genOMNI}
The generalized Omnibus matrix
$\mathfrak{M}\in\R^{mn\times mn}$
is a block matrix satisfying the following conditions:
\begin{enumerate}
    \item Each block entry $\mathfrak{M}_{k,\ell}\in\R^{n\times n},k,\ell\in[m]$ can be written as $\mathfrak{M}_{k,\ell}=\sum_{q=1}^{m}c_{k,\ell}^{(q)}\bA^{(q)},$ where $\{c^{(q)}_{k,\ell}\}_{k,\ell,q}$ satisfy the \textit{convex combination constraints} (CCC), i.e., 
    \begin{enumerate}
        \item $c^{(q)}_{i,j}\in[0,1]$ for all $q,i,j\in [m]$,
        \item  $\sum_{q=1}^m c_{i,j}^{(q)}=1$ for all $i\leq j$.
    \end{enumerate}
    \item For each block row $k\in[m]$ of $\mathfrak{M}$, the cumulative weight of $\mathbf{A}^{(k)}$ is greater than the cumulative weight of the rest of the $\mathbf{A}^{(q)}$, $q\neq k$; i.e., for $q\neq k$
    $$\sum_{\ell} c^{(k,\ell)}_q < \sum_{\ell} c^{(k,\ell)}_k.$$
    \item $\mathfrak{M}$ is symmetric; here enforced by $c^{(q)}_{i,j}=c^{(q)}_{j,i}$ for all $i,j,q$.
\end{enumerate}
\end{defn}
\noindent The classical OMNI matrix of \cite{levin_omni_2017} is a special case of Definition \ref{def:genOMNI} where
$$c_{i,j}^{(q)}=\begin{cases}
\frac{1}{2}&\text{ if }i\neq j\text{ and }q=i\text{ or }q=j;\\
1&\text{ if }i= j=q\\
0&\text{ else}.
\end{cases}$$
Setting $\alpha(k,q)=\sum_{\ell} c^{(q)}_{k,\ell}$ to be the weight put on $\bA^{(q)}$ in the $k$-th block-row of $\mathfrak{M}$, condition 2 above becomes
\begin{align*}
\alpha(k,q)<\alpha(k,k)\hspace{.5cm}\text{ for all $q\neq k$}.
\end{align*}

Let the matrix of row sums be denoted $\mathcal{A}=[\alpha(k,q)]$, and note that $\mathcal{A}\vec{1}=m\cdot \vec{1}$.
Note that the tensor of weights $\mathbf{C}=[C_{k,\ell,q}=c_{k,\ell}^{(q)}]$ then satisfies the following relations.
\begin{itemize}
\item[i.] Partial symmetry: $C_{k,\ell,q}=C_{\ell,k,q}$;
\item[ii.] The $3$-mode vector products (see, for example, \cite{kolda2009tensor} for background on different tensor products) satisfy:
\begin{align*}
\mathbf{C} \,&\overline{\times}_1 \vec{1}=\mathcal{A};\\
\mathbf{C}  \,&\overline{\times}_2 \vec{1}=\mathcal{A};\\
\mathbf{C}  \,&\overline{\times}_3 \vec{1}=\bJ_m.
\end{align*}
\end{itemize}

Once a suitable OMNI matrix $\mathfrak{M}$ has been constructed, the Adjacency Spectral Embedding (ASE) is used to embed the vertices of all graphs simultaneously into a low dimensional Euclidean space.
The ASE introduced in \cite{sussman2014consistent} has proven to be a powerful estimator of latent graph structure, especially in the context of Random Dot Product graphs where the consistency \cite{lyzinski15_HSBM} and asymptotic normality \cite{clt,athreya_survey} of the ASE estimator has been recently established.
The Adjacency Spectral Embedding of a graph into $\mathbb{R}^d$ is defined as follows.
\begin{defn} 
\label{def:ASE}
(Adjacency Spectral Embedding) Let $\text{\textbf{A}} \in \{0,1\}^{n \times n}$ be the adjacency matrix of an $n$-vertex graph.
The \emph{Adjacency Spectral Embedding (ASE)} of $\textbf{A}$ into $\mathbb{R}^d$ is given by
\begin{equation} \label{ase}
   \text{ASE(\textbf{A}}, d) =  \hat{\mathbf{X}} = U_AS_A^{1/2} \in \mathbb{R} ^{n \times d}
\end{equation}
where $ [U_A | \Tilde{U_A}][S_A \oplus \Tilde{S_A}][U_A | \Tilde{U_A}] \hspace{2pt}$ is the spectral decomposition of $|\text{\textbf{A}}| = (\text{\textbf{A}}^T \text{\textbf{A}})^{\frac{1}{2}},$ $S_A$ is the diagonal matrix with the ordered $d$ largest singular values of $\textbf{A}$ on its diagonal, and $U_A \in \mathbb{R}^{n \times d}$ is the matrix whose columns are the corresponding orthonormal eigenvectors of \textbf{A}.
\end{defn}

\noindent Embedding the OMNI matrix $\mathfrak{M}$ then amounts to computing $ASE(\fM,d)$ for a suitable $d$.
There are numerous methods for estimating an optimal embedding dimension $d$ based on subsampling \cite{chen2021estimating,li2020network} and analysis of the spectrum of $\fM$ \cite{chatterjee2014matrix}.
In our experiments and simulations, motivated by \cite{chatterjee2014matrix,zhu2006automatic} we choose the embedding dimension $d$ via an automated elbow-selection procedure applied to the scree plot of $\mathfrak{M}$.

As mentioned previously, the consistency and asymptotic normality of the generalized OMNI embedding is established in \cite{pantazis2022importance} in the Joint Random Dot Product Graph (JRDPG) setting.
\begin{defn}[Joint Random Dot Product Graph of \cite{pantazis2022importance}]
\label{def:JRDPG}
Let $F$ be a distribution on a set $\mathcal{X}\in \R^d $ satisfying $\langle x,x' \rangle \in [0,1]$ for all $x,x'\in \mathcal{X}$.
Let $X_1,X_2,\cdots,X_n\sim F$ be i.i.d. random variables distributed via $F$,
and let $\bP=\bX\bX^T$, where  
$\bX=[X_1^T| X_2^T| \cdots| X_n^T]^T\in\R^{n\times d}$.
The random graphs $(\bA^{(1)},\bA^{(2)},\cdots,\bA^{(m)})$
are distributed as a \emph{Joint
Random Dot Product Graph} (abbreviated JRDPG), written 
$$
(\bA^{(1)},\bA^{(2)},\cdots,\bA^{(m)},\bX)\sim \text{JRDPG}(F,n,m,\bR),
$$ if the following holds:
\begin{itemize}
    \item[i.] Marginally, $(\bA^{(k)},\bX)\sim\mathrm{RDPG}(F,n$) for every $k\in[m]$.  That is, $\bA^{(k)}$ is a symmetric, hollow adjacency matrix with above diagonal entries distributed via
\begin{align}
\label{eq:rdpg}
P(\bA^{(k)}|\bX)=\prod_{i<j}(X_i^TX_j)^{A^{(k)}_{i,j}}(1-X_i^TX_j)^{1-A^{(k)}_{i,j}};
\end{align}
i.e., conditioned on $\bX$ the above diagonal entries are independent Bernoulli random variables with success probabilities provided by the corresponding above diagonal entries in $\bP$.
    \item[ii.] The matrix $\bR\in[-1,1]^{m\times m}$ is symmetric and has diagonal entries identically equal to 1.  We will write the $(k_1,k_2)$-element of $\bR$ via $\rho_{k_1,k_2}$.
    \item[iii.] Conditioned on $\bX$ the collection
    $$\{A^{(k)}_{i,j}\}_{k\in[m],i<j}$$ is mutually independent except that for each $\{i,j\}\in\binom{V}{2}$, we have for each $k_1,k_2\in[m]$,
    $$\mathrm{correlation}(A^{(k_1)}_{i,j},A^{(k_2)}_{i,j})=\rho_{k_1,k_2}.$$
\end{itemize}
\end{defn}
\noindent In the JRDPG model above, the correlation induced by the OMNI embedding can be decomposed into two components: the method induced correlation and the model-inherent correlation.
To wit, we have the following Theorem from \cite{pantazis2022importance}.
\begin{theorem}
\label{thm:pfdiffgenomni}
Let $F$ be a distribution on a set $\mathcal{X}\subset \R^{d}$, where $\langle x, x'\rangle\in[0,1]$ for all $x,x'\in\mathcal{X}$, and assume that $\Delta := \mathbb{E}[X_1 X_1^{T}]$ is rank $d$.
Let $(\bA_n^{(1)},\bA_n^{(2)},\cdots,\bA_n^{(m)},\bX_n)\sim\mathrm{JRDPG}(F,n,m,R)$ be a sequence of correlated $\mathrm{RDPG}$ random graphs.
Further, for each $n\geq 1$, let $\fM_n$ denote the generalized Omnibus matrix as in Definition \ref{def:genOMNI} and $\bhX_{\fM_n}=\mathrm{ASE}(\fM_n,d)$. Consider fixed indices $i\in[n]$ and $s,s_1,s_2\in[m]$, and let $\Big(\hat{\bX}^{(s)}_{\fM_n}\Big)_i$ denote the $i$-th row of the $s$-th $n \times d$ block of $\bhX_{\fM_n}$ (i.e., the estimated latent position from graph $A_n^{(s)}$ of the hidden latent position $X_i$).
There exists a sequence of orthogonal matrices $(\mathcal{Q}_n)_{n=1}^\infty$ such that for all $z \in \R^d$, we have that
    \begin{align}
        \label{eq:corCLT_genomni}
        \lim_{n \rightarrow \infty}
    \p\left[ n^{1/2} \left[ \Big(\bhX^{(s_1)}_{\fM_n} -\bhX^{(s_2)}_{\fM_n}\Big) \mathcal{Q}_n\right]_{i}
    \le z \right]
    = \int_{\text{supp } F} \Phi\left(z, \widetilde\Sigma_{\rho}(x;s_1,s_2)) \right) dF(x), 
    \end{align}
    where
     $\widetilde\Sigma_{\rho}(x;s_1,s_2)$ is given by
\begin{align}
\label{eq:induced_covGEN}
    \widetilde\Sigma_{\rho}(x;s_1,s_2)&=\frac{1}{m^2}\bigg(\underbrace{\sum_{q=1}^{m}\Big(\alpha(s_1,q)-\alpha(s_2,q)\Big)^2}_{\text{method-induced correlation}}\nonumber\\&\hspace{1cm}+\underbrace{2\sum_{q<l}\Big(\alpha(s_1,q)-\alpha(s_2,q))(\alpha(s_1,l)-\alpha(s_2,l)\Big)\rho_{q,l}}_{\text{model-inherent correlation}}\bigg)\Sigma(x)\\
    &=\frac{1}{m^2}\bigg(
    2\sum_{q<l}(\alpha(s_1,q)-\alpha(s_2,q))(\alpha(s_1,l)-\alpha(s_2,l))(\rho_{q,l}-1)\bigg)\Sigma(x).\notag
\end{align}
where $\Sigma(x) 
    := \Delta^{-1} \EX\left[ (x^{T} X_1 - ( x^{T} X_1)^2 ) X_1 X_1^{T} \right] \Delta^{-1}.$
\end{theorem}

For each pair $\{s_1,s_2\}\in\binom{[m]}{2}$ (where $[m]=\{1,2,\cdots,m\}$ and for a set $S$, $\binom{S}{2}$ denotes the collection of all two element subsets of $S$), this Theorem can be used to compute the limit correlation between vertices in graph $\bA^{(s_1)}$ and graph $\bA^{(s_2)}$, denoted 
$\mathfrak{r}_{s_1,s_2}$, can be expressed in terms of the covariance term
\begin{align}
\label{eq:maincorr}
    \mathfrak{r}_{s_1,s_2}  =&\underbrace{1-\frac{\sum_{q=1}^{m}(\alpha(s_1,q)-\alpha(s_2,q))^2}{2m^2}}_{\text{method-induced correlation}}\notag\\
    &- \underbrace{\sql\frac{(\alpha(s_1,q)-\alpha(s_2,q))(\alpha(s_1,\ell)-\alpha(s_2,\ell))\rho_{q,\ell}}{m^2}}_{\text{model-inherent correlation}}\\&=\rho_{me}+\rho_{mo}.\notag
\end{align}

\section{Flat correlation analysis}
\label{sec:flatcorr}
\noindent
Herein, 
we present the flat (induced) correlation problem which arises from the correlation analysis of classical OMNI in  \cite{pantazis2022importance}. 
The main goal is to identify the class of generalized OMNI matrices that induce flat correlation, deduce where the classical OMNI stands within this class, and find the element of this class that induces the minimal flat correlation.

We coin the term flat correlation whenever $\mathfrak{r}_{s_1,s_2}$ is the same across all $s_1,s_2\in[m]$, so that our first task is to understand the following question:  
\begin{itemize}
\item[]\textbf{Flat correlation optimization problem:} If $\rho_{q,l}=\rho$ for all $q,l$, can we find weights $\{c^{(q)}_{i,j}\}_{i,j,q=1}^{m}$ such that $\mathfrak{r}_{s_q,s_l}=\rho$ for all $q,l$?
\end{itemize}
Before analyzing this question further, a bit of analysis is in order.  
Note that if $\rho_{q,l}=\rho$ for all $q,l$, then (writing $\beta_{q,s_1,s_2}:=\alpha(s_1,q)-\alpha(s_2,q)$ to ease notation)
\begin{align*}
    \mathfrak{r}_{s_1,s_2}&=1-\frac{1}{2m^2}\left(\sum_q \beta_{q,s_1,s_2}^2+\sum_{q}\sum_{\ell:\ell\neq q} \rho\beta_{q,s_1,s_2}\beta_{\ell,s_1,s_2}
    \right)\\
    &=1-\frac{1}{2m^2}\left(\left(\sum_q \sqrt{\rho}\beta_{q,s_1,s_2}\right)^2+(1-\rho)\sum_{q} \beta_{q,s_1,s_2}^2
    \right)
\end{align*}
Noting that $\sum_q\alpha(i,q)=m$ for all $i$, we see that $\sum_q\beta_{q,s_1,s_2}=0$, and so
\begin{align}
\label{eq:rho}
    \mathfrak{r}_{s_1,s_2}
    &=1-\frac{1-\rho}{2m^2}\sum_{q} \beta_{q,s_1,s_2}^2
\end{align}
Now, $\sum_{q} \beta_{q,s_1,s_2}^2\leq \sum_q\alpha(s_1,q)^2+\sum_q\alpha(s_2,q)^2\leq 2m^2$.
This gives us that $\mathfrak{r}_{s_1,s_2}\geq \rho$.
Therefore, minimizing the (possible) gap between $\mathfrak{r}_{s_1,s_2}$ and $\rho$ is equivalent to maximizing $\sum_{q} \beta_{q,s_1,s_2}^2$.


\subsection{Weighted OMNI (WOMNI) matrices}
\label{sec:weighted}

In the general OMNI setting, solving the flat correlation problem amounts to optimizing over $m\binom{m}{2}\approx \frac{m^3}{3}$ $c^{(i)}_{j,k}$'s 
with $\Omega(m^2)$ constraints.
Resultingly, deriving sharp correlation bounds on flat correlation in this fully general OMNI setting is notationally and theoretically challenging.
To simplify our subsequent analysis, we will consider minimizing the flat correlation in the space of weighted Omnibus matrices, rather than the space of generalized Omnibus matrices.

The block entries of a Weighted OMNI (WOMNI) are defined via
\begin{align}
\label{defn:womni}
    \mathfrak{M}^W_{i,j}=\begin{cases}
\bA^{(i)}, &\text{ if } i=j \text{ and } \\
c_{i,j}^{(i)}\bA^{(i)}+c_{i,j}^{(j)}\bA^{(j)}, &\text{ if } i\neq j.
\end{cases}
\end{align}
This results in $\binom{m}{2}$ unknowns $\{c^{(i)}_{i,j},c^{(j)}_{i,j}\}_{i<j}$, where the convex combination constraints (CCC) become $c^{(i)}_{i,j}\in[0,1]$ and $c^{(i)}_{i,j}+c^{(j)}_{i,j}=1$ for all $i,j$. 
Under the WOMNI setting, the common value of $\sum_q\beta_{q,s_1,s_2}^2$ becomes 
\begin{align}
\label{eq:WOMNI}
    \sum_{q=1}^{m}&\beta_{q,s_1,s_2}^2=\Big( \alpha(s_1,s_1) - c^{(s_1)}_{s_2,s_1} \Big)^2 + \Big( \alpha(s_2,s_2) - c^{(s_2)}_{s_1,s_2} \Big)^2 + \sum_{q\neq s_1,s_2}\Big(c_{s_1,q}^{(q)} - c_{s_2,q}^{(q)} \Big)^2.
\end{align}
Moreover, in the WOMNI setting, there is a simple relationship between $\mathbf{C}$ and $\mathcal{A}$ defined via
\begin{align}
c^{(k)}_{i,k}&=\alpha(i,k)=1-\alpha(k,i)=c^{(k)}_{k,i} \text{ for }k\neq i,\label{eq:AtoC11}\\
c^{(i)}_{i,k}&=1-\alpha(i,k)=\alpha(k,i)=c^{(i)}_{k,i} \text{ for }k\neq i.\label{eq:AtoC21}
\end{align}

\subsection{Correlation bounds}
\label{ssec:bounds}

As noted above, we are interested in finding a generalized (weighted) Omnibus matrix that solves the flat correlation optimization problem. 
We investigate first whether we can determine bounds for Eq. (\ref{eq:WOMNI}).
We first note that we can obtain a naive lower bound for the flat correlation, that is  \begin{align}
    \label{eq:naivebound}
        \mathfrak{r}_{s_1,s_2}\geq 1-\frac{1-\rho}{2m^2}\left(2\max_q\alpha^2(q,q)+m-2\right).
    \end{align}
Noting that in the weighted OMNI case, the following condition holds (the proof can be found in Appendix \ref{obs:2}) 
\begin{align}
\label{eq:obs2}
    \sum_{q}\alpha(q,q)=\frac{m(m+1)}{2},
\end{align}
we conclude that $\max_q \alpha(q,q)\geq\frac{m+1}{2}$. Next, we present a lower bound for the flat correlation by bounding first the difference $\max_q\alpha(q,q)-\min_q\alpha(q,q)$ which in turn gives us an upper bound for $\max_q\alpha(q,q)$. 
\begin{lemma}
   \label{lem:lowerbound}
\noindent
In the WOMNI setting where $\rho_{s_1,s_2}=\rho$ for all $\{s_1,s_2\}\in\binom{[m]}{2}$, the flat correlation is bounded below as follows
\begin{align}
    \mathfrak{r}_{s_1,s_2}\geq \frac{3}{4}+\frac{\rho}{4}-(1-\rho)\left(\frac{11}{2m}+\frac{24}{m^2}\right).
\end{align}
\end{lemma}
\noindent Recalling that the correlation induced by classical OMNI in this setting is $= \frac{3}{4}+\frac{\rho}{4}$ (see \cite{pantazis2022importance}), we see that 
Lemma \ref{lem:lowerbound} implies that Classical OMNI is asymptotically desirable. The proof (see, Appendix \ref{ssec:lowerbound}) further shows
\begin{enumerate}
    \item $\max_q\alpha(q,q)-\min_q\alpha(q,q)\leq 4.5$,
    \item $\min_q\alpha(q,q)\geq \frac{m-8}{2}$ and $\max_q\alpha(q,q)\leq \frac{m+10}{2}$.
\end{enumerate}
\noindent
In real data applications, oftentimes the sample size of the graphs is moderately small. Thus, the bound in Lemma \ref{lem:lowerbound} may not be as useful when $m$ is small.  However, Lemma \ref{lem:lowerbound} is quite general and can be further tightened under mild assumptions.

Similarly, we can obtain an upper bound for the flat correlation, though this result is subsumed by Theorem \ref{thm:maxOMNI}.
With notation as in Lemma \ref{lem:lowerbound}, the flat correlation is bounded above as follows
\begin{align}
\label{eq:upperbound}
    \mathfrak{r}_{s_1,s_2}\leq \frac{3}{4}+\frac{\rho}{4}+(1-\rho)\left(\frac{5}{m}-\frac{16}{m^2}\right)
\end{align}
From these bounds, we derive the following asymptotic result.
\theorem
\label{thm:womni}
\noindent
With notation as in Lemma \ref{lem:lowerbound}, we have that for any $\{s_1,s_2\}\in\binom{[m]}{2}$, 
\begin{align*}
\Big|\mathfrak{r}_{s_1,s_2}-\frac{3}{4}-\frac{\rho}{4}\Big|=\Theta\Big(\frac{1}{m}\Big).
\end{align*}

\noindent\normalfont
While the above lower and upper bounds are not as informative when $m$ is small, Theorem \ref{thm:womni} indicates that classical OMNI is an appropriate choice when $m$ is large. 

That said, for small $m$ (here $m\geq 3$), classical OMNI can be far from optimal, as shown in the next theorem.
The proof is based on sequential applications of Lagrange multipliers since the system consists of quadratic equations and linear constraints (see Appendix \ref{sec:lagrange} for detail).
\begin{theorem}
\label{thm:maxOMNI}
With notation as in Lemma \ref{lem:lowerbound}, classical OMNI induces the maximum possible flat correlation for all $m\geq 2$; i.e., $\mathfrak{r}_{s_1,s_2}\leq \frac{3}{4}+\rho/4$ for all $m\geq 2$.
\end{theorem}

In the next subsection, we will see a pair of weighted Omnibus matrices that attain the minimum flat correlation (which is $\rho_{me}=\frac{2}{3}$) when $m=3$. Further, one may claim that the flat correlation constraint (that is, $\rho_{me}(s_1,s_2)=C$ for any pair $(s_1,s_2)$, $s_1<s_2$) forces the constraint $\alpha(1,1)=\alpha(2,2)=\cdots=\alpha(m,m)$. That would have allowed us to get much tighter bounds, however, the second example in the next subsection contradicts such claim.
\subsection{Matrices with smaller flat correlation than classical OMNI}
\label{sec:smallcorr}

The initial motivation for conceptualizing and working on the flat correlation optimization problem was to discern whether the classical OMNI induces the minimum flat correlation among all generalized Omnibus matrices. In this subsection, we provide examples of Omnibus matrices that have smaller flat correlation than classical OMNI. We note that the construction of an Omnibus matrix (or a set of Omnibus matrices) which induce smaller correlation for all $m\geq 3$, is non-trivial and further study is needed; see Section \ref{sec:c2o} for steps in this direction. 
\begin{itemize}
\item (The $m=3$ case)
The following pair of Omnibus matrices induce flat correlation $\mathfrak{r}_{i,j}=\frac{2}{3}+\frac{\rho}{3}$.
In Appendix \ref{eg:example}, we show that these matrices attain the optimal minimum flat correlation among all WOMNI matrices. 
The matrices are
\[\mathfrak{M}^W_{3-}=\begin{bmatrix}\bA^{(1)}&\bA^{(1)}&\bA^{(3)}\\\bA^{(1)}&\bA^{(2)}&\bA^{(2)}\\\bA^{(3)}&\bA^{(2)}&\bA^{(3)}\end{bmatrix}\]
and 
\[\mathfrak{M}^W_{3+}=\begin{bmatrix}\bA^{(1)}&\bA^{(2)}&\bA^{(1)}\\\bA^{(2)}&\bA^{(2)}&\bA^{(3)}\\\bA^{(1)}&\bA^{(3)}&\bA^{(3)}\end{bmatrix}.\]
In the $m=3$ case, the quadratic equations in Eq. \ref{eq:WOMNI} are ellipses and their intersection points correspond to the above weighted Omnibus matrices.
\item (The $m=4$ case) Consider the following OMNI matrix that builds upon the $m=3$ case. We want to find constants $a,b\in[0,1]$ such that $a+b=1$ and flat correlation smaller than $\rho_{omni}=\frac{3}{4}$.

\[\mathfrak{M}^W_{4+}=\begin{bmatrix}\bA^{(1)}&\bA^{(2)}&\bA^{(1)}&a\bA^{(1)}+b\bA^{(4)}\\\bA^{(2)}&\bA^{(2)}&\bA^{(3)}&a\bA^{(2)}+b\bA^{(4)}\\\bA^{(1)}&\bA^{(3)}&\bA^{(3)}&a\bA^{(3)}+b\bA^{(4)}\\a\bA^{(1)}+b\bA^{(4)}&a\bA^{(2)}+b\bA^{(4)}&a\bA^{(3)}+b\bA^{(4)}&\bA^{(4)}\end{bmatrix}.\]

\noindent
We show that indeed such $a,b$ exist. In particular, $a=\frac{5-\sqrt{17}}{2}$ and $b=\frac{\sqrt{17}-3}{2}$. The induced flat correlation is 
$$\mathfrak{r}_{i,j}=\frac{4\sqrt{17}-5}{16}+\frac{21-4\sqrt{17}}{16}\rho\approx 0.72+0.28\rho.$$ 
This is an Omnibus matrix that induces smaller flat correlation than classical OMNI (when $\rho<1$), yet there is no indication that this is the optimal solution in the $m=4$ case. 
Putting aside the difficulty of optimizing the flat correlation, does this process of creating Omnibus matrices that produce smaller correlation ($\mathfrak{M}_{3+},\mathfrak{M}_{4+},$ etc.) extends to $m>4$ graphs?
\item (The $m=5$ case) We note that this trend doesn't continue when $m=5$. Specifically, there are no $a,b,c,d\in[0,1],\text{ } a+b=1, \text{ }c+d=1$ such that the following matrix has flat correlation less than $\frac{3}{4}$,
\[\mathfrak{M}^W_{5+}=\begin{bmatrix}\bA^{(1)}&\bA^{(2)}&\bA^{(1)}&a\bA^{(1)}+b\bA^{(4)}&c\bA^{(1)}+d\bA^{(5)}\\\bA^{(2)}&\bA^{(2)}&\bA^{(3)}&a\bA^{(2)}+b\bA^{(4)}&c\bA^{(2)}+d\bA^{(5)}\\\bA^{(1)}&\bA^{(3)}&\bA^{(3)}&a\bA^{(3)}+b\bA^{(4)}&c\bA^{(3)}+d\bA^{(5)}\\a\bA^{(1)}+b\bA^{(4)}&a\bA^{(2)}+b\bA^{(4)}&a\bA^{(3)}+b\bA^{(4)}&\bA^{(4)}&c\bA^{(4)}+d\bA^{(5)}\\c\bA^{(1)}+d\bA^{(5)}&c\bA^{(2)}+d\bA^{(5)}&c\bA^{(3)}+d\bA^{(5)}&c\bA^{(4)}+d\bA^{(5)}&\bA^{(5)}\end{bmatrix}.\]
\end{itemize}


\section{Corr2OMNI optimization problem}
\label{sec:c2o}

In this section, our goal will be to derive OMNI weights $\mathbf{C}$ such that we can induce a correlation structure as close to the latent correlation $\{\rho_{i,j}\}$ as feasible.
Letting $\mathbf{R}\in\mathbb{R}^{m\times m}$ be the symmetric matrix defined via
$$\bR_{i,j}=\begin{cases}
    \rho_{i,j}&\text{ if }i\neq j\\
    1&\text{ if }i=j.
\end{cases}
$$
Assuming for the moment that $\mathbf{R}$ is positive-definite, let the Cholesky decomposition of $\mathbf{R}$ be given by $\mathbf{R}=\mathbf{L}\mathbf{L}^T$.
The correlation in Eq. \ref{eq:maincorr} can then be written as (where $\beta_{:,s_1,s_2}$ is the column vector with $i$-th entry equal to $\beta_{i,s_1,s_2}$)
\begin{align*}
 \mathfrak{r}_{s_1,s_2}&=1-\frac{1}{2m^2}\beta_{:,s_1,s_2}^T\mathbf{R}\beta_{:,s_1,s_2}\\
 &=1-\frac{1}{2m^2}(\beta_{:,s_1,s_2}\mathbf{L})(\beta_{:,s_1,s_2}\bL)^T \\
 &=1-\frac{1}{2m^2}\|\mathcal{A}(s_1,:)\mathbf{L}-\mathcal{A}(s_2,:)\mathbf{L}\|_2^2\\
 &=1-\frac{1}{2m^2}\|(\mathcal{A}\bL)(s_1,:)-(\mathcal{A}\bL)(s_2,:)\|_2^2
\end{align*}
Letting $\fR$ be symmetric matrix defined via
$$\fR_{i,j}=\begin{cases}
    \fr_{i,j}&\text{ if }i\neq j\\
    1&\text{ if }i=j,
\end{cases}
$$
and letting $\tilde{\mathcal{A}}=\mathcal{A}\mathbf{L}$, 
then yields (where $\bJ_m=\mathds{1}_m\mathds{1}_m^T$)
\begin{align*}
\fR=\bJ_m-\frac{1}{2m^2}\mathbf{D}^{(2)},
\end{align*}
where $\bD^{(2)}$ is an $m\times m$ (squared) distance matrix with entries $D^{(2)}_{ij}=\|\tilde{\mathcal{A}}(i,:)-\tilde{\mathcal{A}}(j,:)\|_2^2$.
Ideally, we have that $\fR\approx\bR$, which leads us to our 
Corr2OMNI 
optimization problem
\begin{itemize}
\item[]\textbf{Corr2OMNI optimization problem:} With notation defined as above, find weights $\{c^{(q)}_{i,j}\}_{i,j,q=1}^{m}$ that minimize 
the difference between $\fR$ and $\bR$; here posited as finding weights that minimize the following \textit{stress function}, 
where $\Delta=\left(2m^2(\bJ_m-\bR)\right)^{\circ \frac{1}{2}}$ (here $^{\circ \frac{1}{2}}$ denotes the square root operation done entry-wise) and $d_{i,j}(\tcA)=\|\tcA(i,:)-\tcA(j,:)\|_2$, 
$$\sigma(\tcA)=
\sum_{i<j} w_{ij}(\delta_{i,j}-d_{i,j}(\tcA))^2.
$$
\end{itemize}
We will tackle this problem in two stages: first, we aim to find a feasible $\mathcal{A}$ matrix (which we can restrict to the class of WOMNI configurations), and then find the corresponding weights $\{c^{(q)}_{i,j}\}_{i,j,q=1}^{m}$.

For the first part of the problem, we will seek to minimize stress over $\tilde{\mathcal{A}}=\mathcal{A}\mathbf{L}$, where the constraints on $\mathcal{A}$ translate to constraints on $\tilde{\mathcal{A}}=\mathcal{A}\mathbf{L}$ via
\begin{enumerate}[label=(\roman*)]
    \item 
    $\mathcal{A}\mathds{1}_m=m\mathds{1}_m,\Leftrightarrow \tcA\bL^{-1}\mathds{1}_m=m\mathds{1}_m$;
        \item For each $i\in[m]$, $e_i^T \mathcal{A} e_i > e_i^T\mathcal{A}e_j\Leftrightarrow e_i^T\tcA\bL^{-1}(e_i-e_j)>0$ for every $j\neq i$;
    \item $\mathcal{A}\geq 0\Leftrightarrow \tcA(\bL)^{-1}\geq 0$ (entry-wise);
    \item For all $i\neq j$, $\cA_{ij}+\cA_{ji}=1\Leftrightarrow \sum_k \tcA_{ik}(\bL^{-1})_{kj}+\tcA_{jk}(\bL^{-1})_{ki}=1$ (if restricting to the WOMNI setting).
\end{enumerate} 
We adopt a simple stress majorization approach for approximately solving the constrained Multidimensional Scaling problem posed above.
We first (using the notation of \cite{borg2005modern}) write the iterative majorization of the stress function via (where $\tcA^{(t)}$ is the approximate solution at the $t$-th step in the algorithm and $c_{\Delta,\mathbf{W}}$ is a constant depending on $\Delta$ and $\mathbf{W}$)
\begin{align*}
\sigma(\tcA^{(t)})&= c_{\Delta,\mathbf{W}}+
\sum_{i<j} w_{ij}d^2_{ij}(\tcA^{(t)})-2\sum_{i<j}(w_{ij}\delta_{ij})d_{ij}(\tcA^{(t)})\\
&\leq 
c_{\Delta,\mathbf{W}}+
\sum_{i<j} w_{ij}d^2_{ij}(\tcA^{(t)})-2
\text{tr}\left((\tcA^{(t)})^T
\mathbf{B}(\tcA^{(t-1)} )\tcA^{(t-1)}\right)=\tau(\tcA^{(t)},\tcA^{(t-1)})
.
\end{align*}
where 
\begin{align*}
\mathbf{B}(\tcA^{(t-1)} )_{ij}&=\begin{cases}
-\frac{w_{ij}\delta_{ij}}{d_{ij}(\tcA^{(t-1)})}&\text{ if }i\neq j\text{ and }d_{ij}(\tcA^{(t-1)})\neq 0\\
0&\text{ if }i\neq j\text{ and }d_{ij}(\tcA^{(t-1)})= 0
\end{cases}\\
\mathbf{B}(\tcA^{(t-1)} )_{ii}&=-\sum_{j\neq i} \mathbf{B}(\tcA^{(t-1)} )_{ij}
\end{align*}
We note that $\tau(\tcA^{(t)},\tcA^{(t-1)})$ is a quadratic function of $\tcA^{(t)}$ and all constraints on $\tcA$ are affine, hence we can use an off-the-shelf Quadratic Program solver (here \texttt{piqp} in \texttt{R} from \cite{schwan2023piqp}) to approximately, iteratively, minimize the majorizing function with the given constraints until a stopping criteria is met (either an iteration max or the obtained stress $|\sigma(\tcA^{(t)})-\sigma(\tcA^{(t-1)})|<\epsilon$ for a preset threshold $\epsilon$).
We note here that the algorithm at present often iteratively improves the stress even after iterates with very little change, and it is recommended that a large number of iterates be run rather than a preset $\epsilon$ be selected.
For the $m$ considered herein ($m\leq30$), a large number of iterates can be run in a short time window.

Once the approximating $\tcA$ is found in the WOMNI setting, we can convert back to $\cA$ via $\cA=\tcA\bL^{-1}$.
We can then set (where all other entries of the weight tensor $\mathbf{C}$ are 0 by the WOMNI restrictions)
\begin{align}
c^{(k)}_{i,k}&=\cA(i,k)=1-\cA(k,i)=c^{(k)}_{k,i} \text{ for }k\neq i\label{eq:AtoC1}\\
c^{(i)}_{i,k}&=1-\cA(i,k)=\cA(k,i)=c^{(i)}_{k,i} \text{ for }k\neq i.\label{eq:AtoC2}
\end{align}
Deriving the analogous equations for $\mathbf{C}$ in terms of $\mathcal{A}$ in the general Omnibus setting seems nontrivial, and we do not pursue this further here.

\section{Experiments and Simulations}
\label{sec:exp}

We next consider simulations and a trio of real data experiments to determine the effectiveness of our \texttt{Corr2Omni} procedure\footnote{Note that all code needed to produce the experiments and simulations can be found at \url{https://www.math.umd.edu/~vlyzinsk/Corr2Omni/}}.

\subsection{Generating Flat Correlation with Corr2Omni}
\label{sec:fc2o}

A natural application of the corr2Omni framework is to attempt to generate weight matrices $\mathcal{A}$ that induce minimal flat correlation.  
For example, if we set $\mathbf{R}=\mathbf{I}_3$ and 
$\mathfrak{R}=\frac{2}{3}\bJ_3+\frac{1}{3}\mathbf{I}_3$
(the minimal flat correlation in the $m=3$ setting), the output of our \texttt{corr2Omni} procedure is (rounded to 4 digits) exactly 
$$
\begin{bmatrix}
2 &1 &0\\
0 &2&1\\
1&0&2
\end{bmatrix},
$$
and $\fM_{3-}^{W}$ is recovered.
Note that we are \textbf{not} enforcing the flat correlation constraints implicitly in the \texttt{corr2Omni} function.
In the $m=4$ case, we do not know the optimal flat correlation.
If we use the optimal correlation from the $m=3$ case, setting $\mathbf{R}=\mathbf{I}_4$ and 
$\mathfrak{R}=\frac{2}{3}\bJ_4+\frac{1}{3}\mathbf{I}_4$, \texttt{corr2Omni} estimates (rounded to 4 digits)
$$\mathcal{A}=
\begin{bmatrix}
2.4259 &0 &1& 0.5741\\
1& 2.4259 &0& 0.5741\\
0& 1& 2.4259 &0.5741\\
0.4259 &0.4259 &0.4259 &2.7222
\end{bmatrix};\quad \mathcal{A}_{\fM_{4-}^W}\begin{bmatrix}
2.4384 &1 &0& 0.5616\\
0& 2.4384 &1& 0.5616\\
1& 0& 2.4384 &0.5616\\
0.4384 &0.4384 &0.4384 &2.6847
\end{bmatrix}
$$
and we note the similarity to the weight matrix constructed in Section \ref{sec:smallcorr} (with the role of $\bA^{(2)}$ and $\bA^{(3)}$ flipped).
The correlation matrix induced by the \texttt{corr2Omni} construction here is (rounded to 4 digits)
$$ 
\begin{bmatrix}
1& 0.7213& 0.7213& 0.7148\\
0.7213& 1 &0.7213& 0.7148\\
0.7213& 0.7213& 1&0.7148\\
0.7148& 0.7148& 0.7148& 1
\end{bmatrix},
$$
which, while not exactly flat, lends credence to the potential that the $m=4$ construction in Section \ref{sec:smallcorr} is also optimal and induces the smallest flat correlation for $m=4$.
For $m=5$ the corresponding constructions are (when using the targets $\mathbf{R}=\mathbf{I}_5$ and 
$\mathfrak{R}=\frac{2}{3}\bJ_5+\frac{1}{3}\mathbf{I}_5$)
\begin{align*}
\mathcal{A}=\begin{bmatrix}
3& 0& 0& 1&1\\
1& 3 &0& 0&1\\
1& 1& 3&0&0\\
0& 1& 1& 3&0\\
0& 0&1&1&3
\end{bmatrix},\quad 
\mathfrak{R}=\begin{bmatrix}
1.00& 0.72& 0.68& 0.68& 0.72\\
0.72& 1.00& 0.72& 0.68& 0.68\\
0.68& 0.72& 1.00& 0.72& 0.68\\
0.68& 0.68& 0.72& 1.00& 0.72\\
0.72& 0.68& 0.68& 0.72& 1.00
\end{bmatrix}
\end{align*}
In this same setting, if we use the targets $\mathbf{R}=\mathbf{I}_5$ and 
$\mathfrak{R}=0.72\bJ_5+0.28\mathbf{I}_5$ (0.72 an educated lower bound from the first stage estimate), we arrive at (where we round to 4 digits for ease of presentation)
\begin{align*}
\mathcal{A}=\begin{bmatrix}
3.000 &0.344 &0.000 &1.000 &0.656\\
0.656 &3.000 &0.344 &0.000 &1.000\\
1.000 &0.656 &3.000 &0.344 &0.000\\
0.000 &1.000 &0.656 &3.000 &0.344\\
0.344 &0.000 &1.000 &0.656 &3.000
\end{bmatrix},\quad 
\mathfrak{R}=\begin{bmatrix}
1.000 &0.724 &0.721 & 0.721 & 0.724\\
0.724 &1.000 &0.724 & 0.721 & 0.721\\
0.721 &0.724 &1.000 & 0.724 & 0.721\\
0.721 &0.721 &0.724 & 1.000 & 0.724\\
0.724 &0.721 &0.721 & 0.724 & 1.000
\end{bmatrix}
\end{align*}
and approximately flat correlation has been induced.  This suggests the two stage procedure for estimating the optimal flat correlation: 
first, use the estimate from a lower $m$, and then rerun \texttt{corr2Omni} using an updated estimate for the correlation level obtained from the first iterate.

\begin{figure}[t!]
  \centering
\includegraphics[width=1\linewidth]{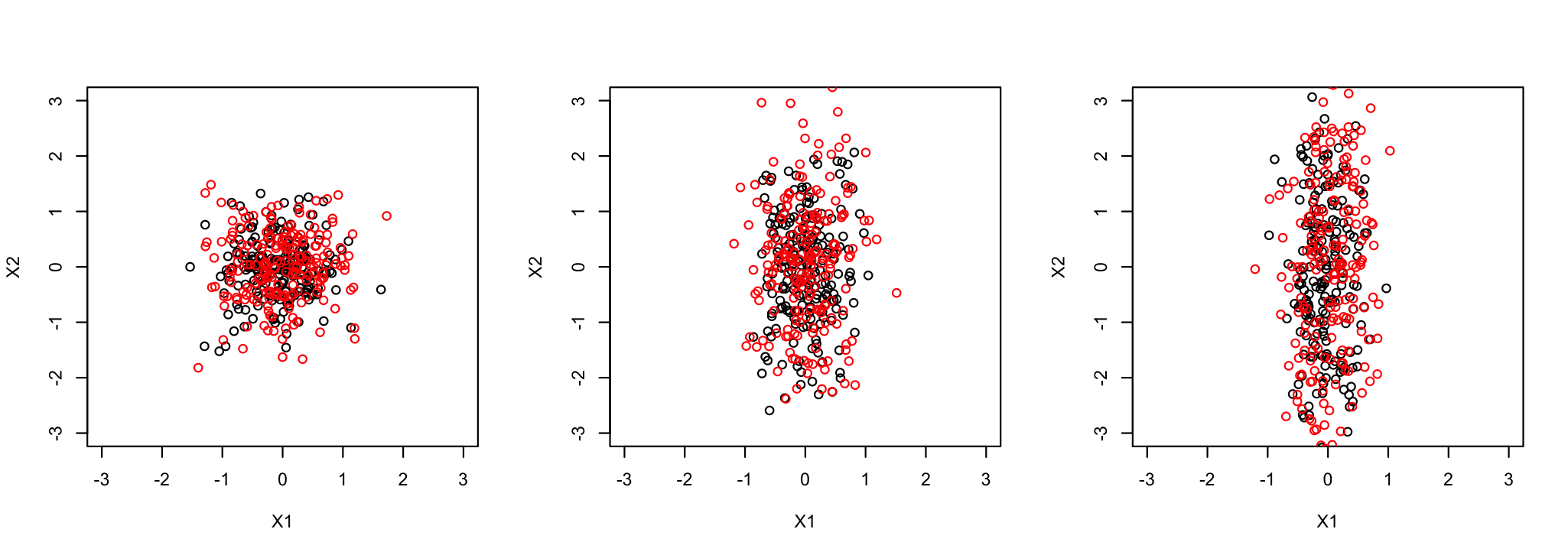}
  \caption{
  We sample $nMC=200$ graph triplets from the $\mathrm{JRDPG}_{\text{gen}}(F,500,3,\rho\mathds{1}_3)$ distribution where $F$ is specified in Section \ref{sec:simFlat} and $\rho=0\,(L),\,0.25\,(C),\,0.5\, (R)$.
  In each panel, we plot $\sqrt{500}(\bhX_1-\bhX_{501})$ where $\bhX$ is the ASE embedding of classical OMNI (in black) and $\fM_{3-}^{W}$ (in red).
  }
  \label{fig:flatcor3}
\end{figure}

\subsection{Flat Correlation Simulations}
\label{sec:simFlat}

A natural model for sampling random graphs with fixed inherent correlation $\rho_{s_1,s_2}=\rho$ for all $\{s_1,s_2\}\in\binom{[m]}{2}$ is the Single Generator Model from \cite{pantazis2022importance}.  To wit, the model is as follows.
\begin{definition}\label{def:generator_RDPG}\emph{(Single Generator Model from \cite{pantazis2022importance})}
With notation as in Definition \ref{def:JRDPG}, 
 we say that the random graphs $A^{(0)},A^{(1)},\cdots,A^{(m)}$ are an instantiation of a multiple $\mathrm{RDPG}$ with generator matrix $A^{(0)}$ and we write  $(A^{(0)},A^{(1)},A^{(2)},\cdots,A^{(m)},\bX)\sim \mathrm{JRDPG}_{\text{gen}}(F,n,m,\nu)$ if 
\begin{itemize}
\item[i.] Marginally, $(A^{(0)},\bX)\sim\mathrm{RDPG}(F,n$);
\item[ii.] $\nu\in[0,1]^{m}$ is a nonnegative vector with entries in $[0,1]$; we denote the $k$-th entry of $\nu$ via $\nu_k=\varrho_{k}$.  
\item[iii.] $(A^{(1)},A^{(2)},\cdots,A^{(m)},\bX)\sim \mathrm{JRDPG}(F,n,m,R)$ where 
$R=\nu\nu^T+\mathrm{diag}(I_{m}-\nu\nu^T)$ so that 
for $1\leq k_1<k_2\leq m$,
$$R_{k_1,k_2}
:=\rho_{k_1,k_2}
=\varrho_{k_1}\varrho_{k_2}.$$
\end{itemize}
\end{definition}
\noindent 
Note that if $\nu=\varrho\mathds{1}_m$ is a constant vector in the $\mathrm{JRDPG}_{\text{gen}}$ model, then the inherent correlation is flat (equal to $\varrho^2$).

To build further intuition, we consider the following simple experiment.
We sample $nMC=200$ graph sequences from the $\mathrm{JRDPG}_{\text{gen}}(F,500,3,\rho\mathds{1}_3)$ distribution, where sampling from $F$ consists of creating i.i.d. Dirichlet(1,1,1) random variables, and considering only the first two dimensions (so that $d=2$ here).
In Figure \ref{fig:flatcor3}, we show results for $\rho=0\,(L),\,0.25\,(C),\,0.5\,(R)$, where in
each panel, we plot $\sqrt{500}(\bhX_1-\bhX_{501})$ where $\bhX$ is the ASE embedding of classical OMNI (in black) and $\fM_{3-}^{W}$ (in red).
In each setting, we estimated the Gaussian covariance matrix (assuming diagonal covariance structure), and display the covariance estimates below
\begin{align*}
&\quad\rho=0\quad\quad\quad\quad\rho=0.25\quad\quad\quad\quad\rho=0.5\\
\text{Classical OMNI}\quad
&\begin{bmatrix}
0.26 & 0\\
0 & 0.31
\end{bmatrix}\hspace{5mm}
\begin{bmatrix}
0.16 & 0\\
0 & 0.98
\end{bmatrix}\hspace{5mm}
\begin{bmatrix}
0.11 & 0\\
0 & 2.02
\end{bmatrix}\\
\text{Embedding }\fM_{3-}^{W}\quad
&\begin{bmatrix}
0.34 & 0\\
0 & 0.40
\end{bmatrix}\hspace{5mm}
\begin{bmatrix}
0.22 & 0\\
0 & 1.29
\end{bmatrix}\hspace{5mm}
\begin{bmatrix}
0.16 & 0\\
0 & 3.06
\end{bmatrix}
\end{align*}
Note the smaller covariance matrix (due to more correlation being induced across networks) in classical OMNI as predicted by the construction.
Indeed, considering the $\rho=0$ case, we have that the limit correlations are a constant factor of each other and we note
$$
\begin{bmatrix}
0.26 & 0\\
0 & 0.31
\end{bmatrix}^{1/2}\cdot\frac{3/2}{4/3}=
\begin{bmatrix}
0.57 & 0\\
0 & 0.62
\end{bmatrix}\approx
\begin{bmatrix}
0.34 & 0\\
0 & 0.40
\end{bmatrix}^{1/2}=
\begin{bmatrix}
0.59 & 0\\
0 & 0.63
\end{bmatrix}
$$ 
Note also the effect of the inherent correlation on the limiting covariance structure, as the covariance entries rise (as expected) significantly in the presence of inherent model correlation.


\begin{figure}[t!]
  \centering
\includegraphics[width=0.8\linewidth]{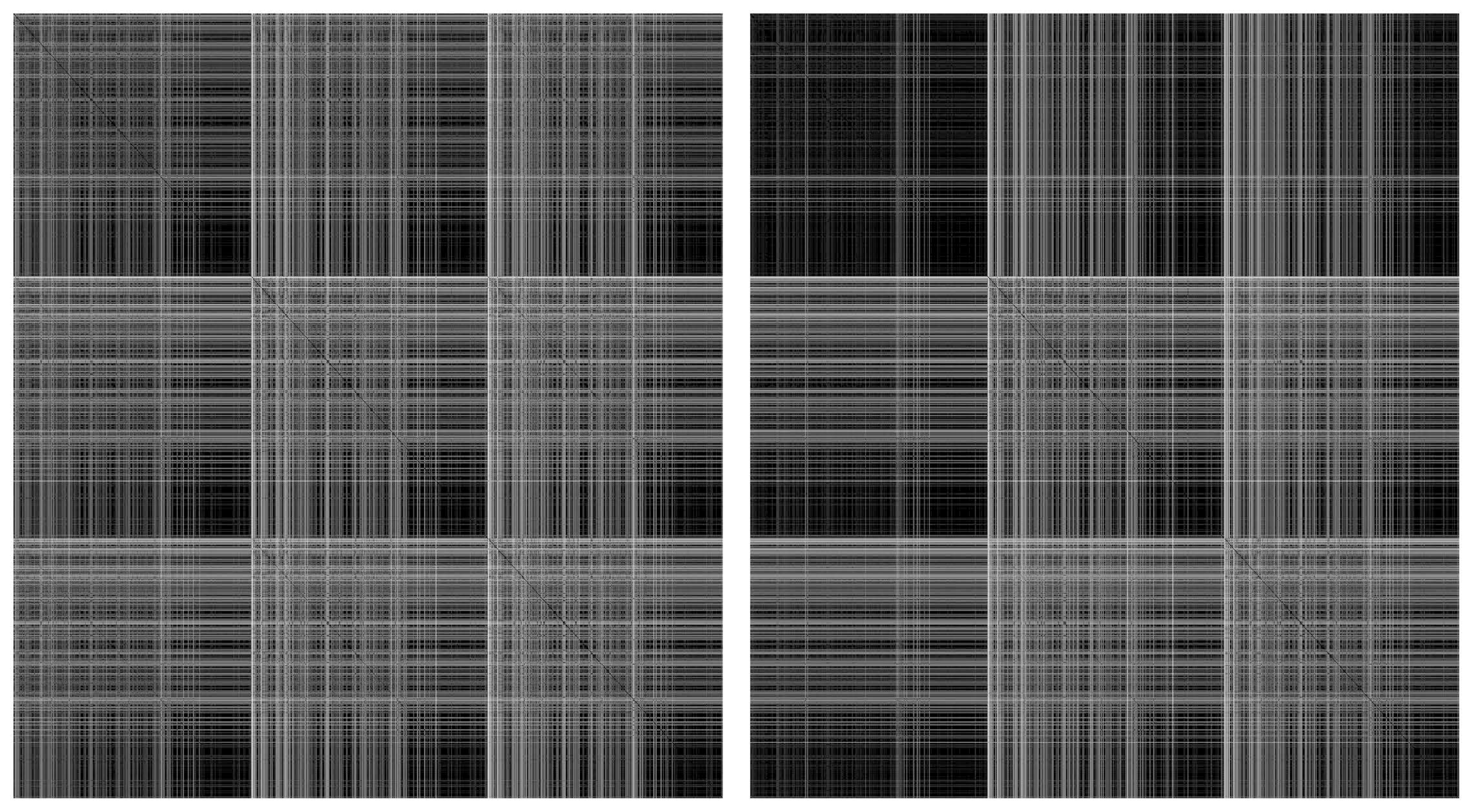}
  \caption{
Vertex--to--vertex Euclidean distance matrices computed from the classical Omni embedding (L) and the \texttt{corr2Omni} embedding (R) in the social network setting of Section \ref{sec:soc}.
}
  \label{fig:sn}
\end{figure}

\subsection{Multilayer social network}
\label{sec:soc}

Classical OMNI has proven to be a powerful joint embedding procedure, and algorithms incorporating classical OMNI outperform competing approaches across a host of inference tasks \cite{levin2017central}.
That said, classical OMNI does have the tendency to over-correlate networks in the embedding space.
To demonstrate this, we consider the multilayer social network from \cite{5992579}.
The network contains three layers showing the social connections among a common set of users in Twitter (now X), Youtube and Friendfeed.
After binarizing and symmetrizing the networks, removing isolated vertices and keeping only vertices common to all three networks, we are left with three graphs on a common set of 422 vertices.
In order to show the effect of \texttt{corr2Omni}, we first estimate the edge-wise correlation matrix for each pair of graphs using the \texttt{alignment strength} machinery of \cite{fishkind2019alignment}.
Alignment strength provides a principled measure of correlation across graphs that accounts for both edge-wise correlation and topological graph heterogeneity.
Here the alignment strength between a pair of graphs is computed as follows: (where $\Pi_n$ is the set of $n\times n$ correlation matrices, and $\bA^{(i)}_H$ is the matrix $\bA^{(i)}$ with its diagonal entries set to 0):
\begin{align*}
\text{str}&(\bA^{(i)},\bA^{(j)})=
1-\frac{\|\bA^{(i)}-\bA^{(j)}\|_F^2 }{\frac{1}{n!}\sum_{P\in\Pi_n}\|\bA^{(i)}P-P\bA^{(j)}\|_F^2 }\\
&=
1-\frac{\|\bA^{(i)}-\bA^{(j)}\|_F^2 }{
\|\bA^{(i)}\|_F^2+\|\bA^{(j)}\|_F^2-2\sum_\ell\bA^{(i)}_{\ell,\ell}\sum_\ell\bA^{(j)}_{\ell,\ell}/n-\sum_{\ell,k}(\bA^{(i)}_H)_{\ell,k}\sum_{\ell,k}(\bA^{(j)}_H)_{\ell,k}/\binom{n}{2}}
\end{align*}
We use this target $\mathfrak{R}$ (where $\mathfrak{R}_{ij}=\text{str}(\bA^{(i)},\bA^{(j)})$) to estimate the weight matrix 
$\mathcal{A}$ using this $\mathfrak{R}$ and $\mathbf{R}=\mathfrak{R}$.

\begin{figure}[t!]
  \centering
\includegraphics[width=0.8\linewidth]{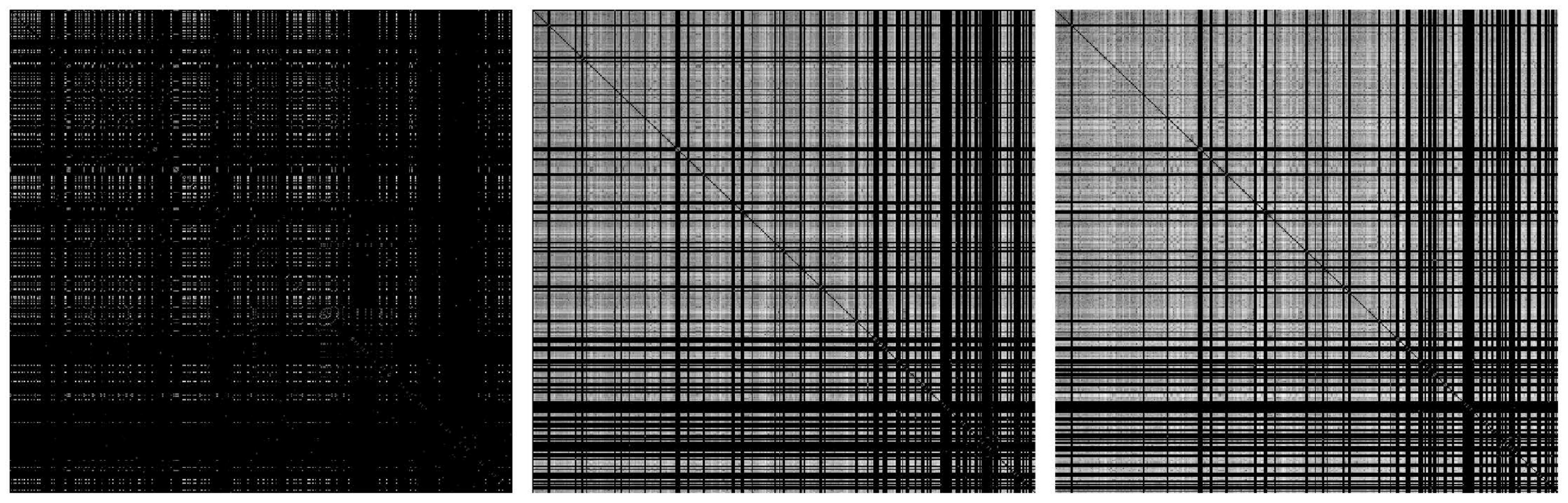}
  \caption{
For the social networks considered in Section \ref{sec:soc}, we plot row-wise Euclidean distance matrices computed for $\bA^{(1)}$ (L), $\bA^{({}2)}$ (C), and $\bA^{(3)}$ (R) (so that $\mathbf{D}^{(i)}_{k,\ell}=\|\bA^{(i)}_{k,\bullet}-\bA^{(i)}_{\ell,\bullet}\|_F$).
}
  \label{fig:sn2}
\end{figure}

If we consider the embedding correlation induced by \texttt{corr2Omni} as $\hat{\mathfrak{R}}_{c2O}$ and by classical OMNI as $\hat{\mathfrak{R}}_{O}$, we have 
$$    
\mathfrak{R}= \begin{pmatrix}  
 1.00& 0.28& 0.28\\
 0.28& 1.00& 0.55\\
 0.28& 0.55& 1.00
 \end{pmatrix};\,\,
\hat{\mathfrak{R}}_{O} \begin{pmatrix}  
 1.00& 0.92& 0.92\\
 0.92& 1.00& 0.95\\
 0.92& 0.95& 1.00
 \end{pmatrix};\,\,
\hat{\mathfrak{R}}_{c2O} \begin{pmatrix}  
 1.00& 0.68& 0.73\\
 0.68& 1.00& 0.95\\
 0.73& 0.95& 1.00
\end{pmatrix}
$$
and that
$$
\|\hat{\mathfrak{R}}_{c2O}-\mathfrak{R}\|_F=1.02;\quad \|\hat{\mathfrak{R}}_{O}-\mathfrak{R}\|_F=1.40
$$ 
While both methods tend to induce more correlation than is present in $\mathfrak{R}$, classical OMNI tends to flatten the correlation structure in $\mathfrak{R}$; by flatten, here we mean that the ratios of the correlations in the off-diagonal blocks is closer to 1 in $\hat{\mathfrak{R}}_{O}$ than in $\hat{\mathfrak{R}}_{c2O}$ and $\mathfrak{R}$.
In the embedding space, the effect of this is signal from $\bA^{(2)}$ (Friendfeed) and $\bA^{(3)}$ (Twitter) bleeds into the structure of $\bA^{(1)}$ (Youtube).  
To see this, we embed the two Omnibus matrices (where the embedding dimensions are chosen as in Section \ref{sec:omni}), and plot the vertex-vertex distance matrices for classical Omni and for \texttt{corr2Omni} in Figure \ref{fig:sn} (classical Omni on the left and \texttt{corr2Omni} on the right).
Each figure is a distance matrix in $\mathbb{R}^{1266\times1266}$; comparing these to the row-wise Euclidean distance matrices computed for $\bA^{(1)}$, $\bA^{(2)}$, and $\bA^{(3)}$ in Figure \ref{fig:sn2} (so that $\mathbf{D}^{(i)}_{k,\ell}=\|\bA^{(i)}_{k,\bullet}-\bA^{(i)}_{\ell,\bullet}\|_F$).
Comparing Figures \ref{fig:sn} and \ref{fig:sn2}, we see the structural bleed from $\bA^{(2)}$ and $\bA^{(3)}$ into $\bA^{(1)}$.
In the next examples, we will explore how this flattening of the structure caused by classical OMNI can impact subsequent inference.


\begin{figure}[t!]
  \centering
\includegraphics[width=1\linewidth]{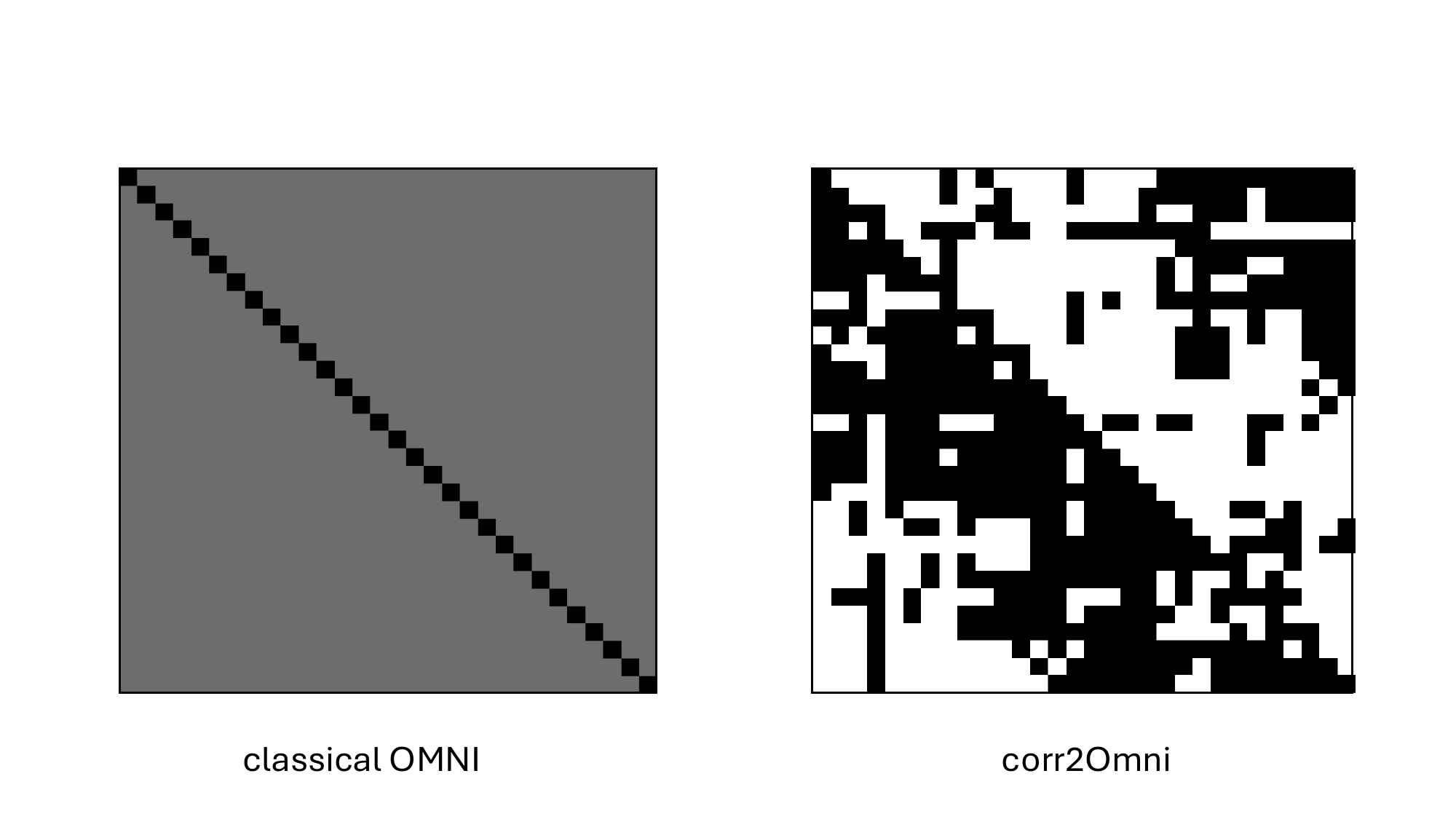}
  \caption{
Detail of the obtained $\mathcal{A}$ matrix obtained by \texttt{corr2Omni} (R) and classical OMNI (L) in the DTMRI experiment.
Note that the diagonal has been 0'd out to allow for the off-diagonal detail to be observed.
Large values ($\leq1$) are lighter grey and, where black denotes a value of 0 and white a value of 1.
}
  \label{fig:c2o_mri2}
\end{figure}

\subsection{DTMRI connectomes}
\label{sec:mri}

We next consider the test-retest DTMRI data from the HNU1 dataset \cite{hnu12,hnu11}, downloaded here from \url{neurodata.io}.
These networks (the \texttt{Desikan\_space-MNI152NLin6\_res-2x2x2} brains in the HNU1 dataset on \url{neurodata.io}) are 70 vertex graphs from 10 individuals.
Each vertex represents a brain region in the Desikan atlas and the edges are weighted to measure the level of neural connectivity between regions.
From each patient, we use 3 brain scans (from the 10 available, we use scan 1,2, and 3) for a total of 30 networks.
This down-sampling was purely for algorithmic scalability, as the \texttt{corr2Omni} algorithm, as written, is computationally expensive to run on 100 networks.

\begin{figure}[t!]
  \centering
\includegraphics[width=1\linewidth]{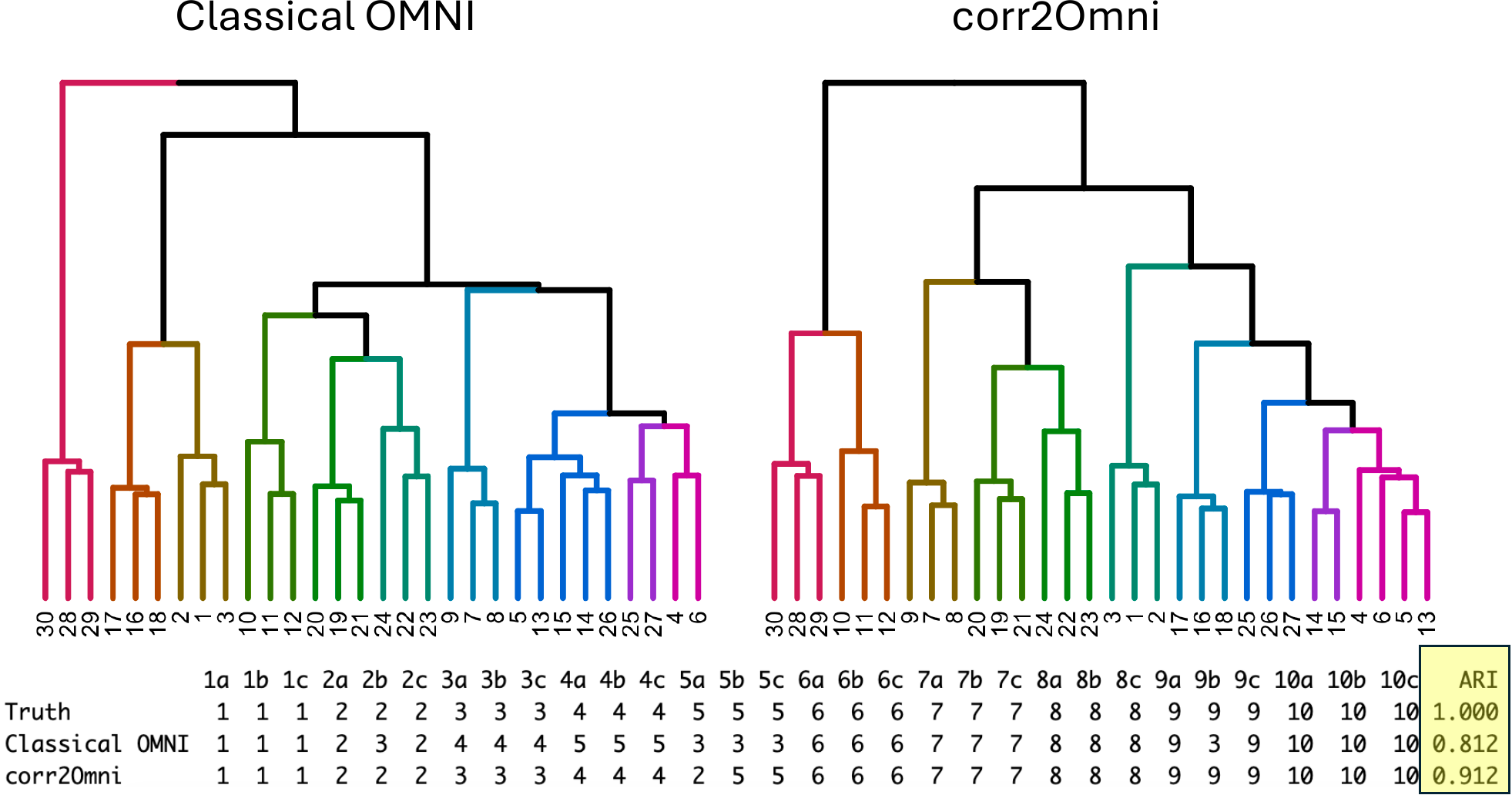}
  \caption{
For the 30 DTMRI networks, we apply the \texttt{corr2Omni} procedure (in the WOMNI setting) to find the weight tensor $\mathbf{C}$ that aims to induce correlation 
$\mathfrak{R}=0.54 \cdot \bJ_{30}+0.46 \cdot \mathbf{I}_{30}.$
We embed the collection using both classical OMNI and the \texttt{corr2Omni} weights.
Once the collection networks are embedded, we compute the Frobenius norm distance between each graph pair in the embedding and use hierarchical clustering to cluster the collection of 30 graphs into 10 clusters.  The clustering dendrograms for classical OMNI (L) and \texttt{corr2Omni} (R) are shown.
Lastly, these cluster labels are compared (using Adjusted Rand Index or ARI) to the true patient labels.
  }
  \label{fig:c2o_mri}
\end{figure}

Treating the patient labels as missing, we plan to embed these 30 networks using the OMNI procedure, and here aim to have the embedding induce the minimal flat correlation in the embedding space.
As the patient labels are treated as unknown, this choice of embedding is reasonable in that it would maximally preserve the individual latent structure between each brain pair, ideally preserving the similarity of scans within a patient and the differences across patients.
Once embedded, we will then cluster the collection of graphs to attempt to recover the missing patient labels.
For these 30 networks, the minimal induced flat correlation is lower bounded by Lemma \ref{lem:lowerbound} (in the case where $\rho=0$) by $0.54$.
We then apply the \texttt{corr2Omni} procedure (in the WOMNI setting) to find the weight tensor $\mathbf{C}$ that aims to induce correlation 
$\mathfrak{R}=0.54 \cdot \bJ_{30}+0.46 \cdot \mathbf{I}_{30}.$
Once the weight matrix $\mathcal{A}$ is obtained (with markedly lower stress value than classical OMNI; 22700.61 versus 24873.56), we invert $\mathcal{A}$ to find the weight tensor $\mathbf{C}$ according to Eqs. \ref{eq:AtoC1}--\ref{eq:AtoC2}, and embed the weighted Omnibus matrix $\mathfrak{M}$ into $d=8$ dimensions using ASE, where the embedding dimension was chosen via an automated elbow-selection procedure of Section \ref{sec:omni}.
We also embed the networks using classical OMNI, also into $d=8$ dimensions as selected by an independent analysis of the scree plot of the classical Omnibus matrix.

Once the collection networks are embedded, as in \cite{pantazis2022importance} we compute the Frobenius norm distance between each graph pair in the embedding, and use hierarchical clustering to cluster the collection of 30 graphs into 10 clusters; the corresponding dendrograms are plotted in Figure \ref{fig:c2o_mri}.
Lastly, these cluster labels are compared (using Adjusted Rand Index or ARI) to the true patient labels; see the table in Figure \ref{fig:c2o_mri}.
From Figure \ref{fig:c2o_mri2}, we note that the weight matrix obtained by \texttt{corr2Omni} is significantly more heterogeneous than the flat $\mathcal{A}=0.5\cdot\bJ_{30}+15\cdot\mathbf{I}_{30}$ in classical OMNI.
While classical OMNI performs fairly well in recovering the latent patient labels, the heterogeneous $\mathcal{A}$ of \texttt{corr2Omni} better decorrelates the networks (demonstrated by the reduction of stress versus classical OMNI), and this leads to the patient labels being recovered better in the \texttt{corr2Omni} setting (as measured by the higher fidelity clustering ARI).
In settings where the structure present in a collection of graphs is unknown, this example points to the utility of the \texttt{corr2Omni} embedding for better preserving the latent structure versus classical OMNI.

\subsection{Analyzing Aplysia motor programs}
\label{sec:aplysia}

\begin{figure}[t!]
  \centering
\includegraphics[width=1\linewidth]{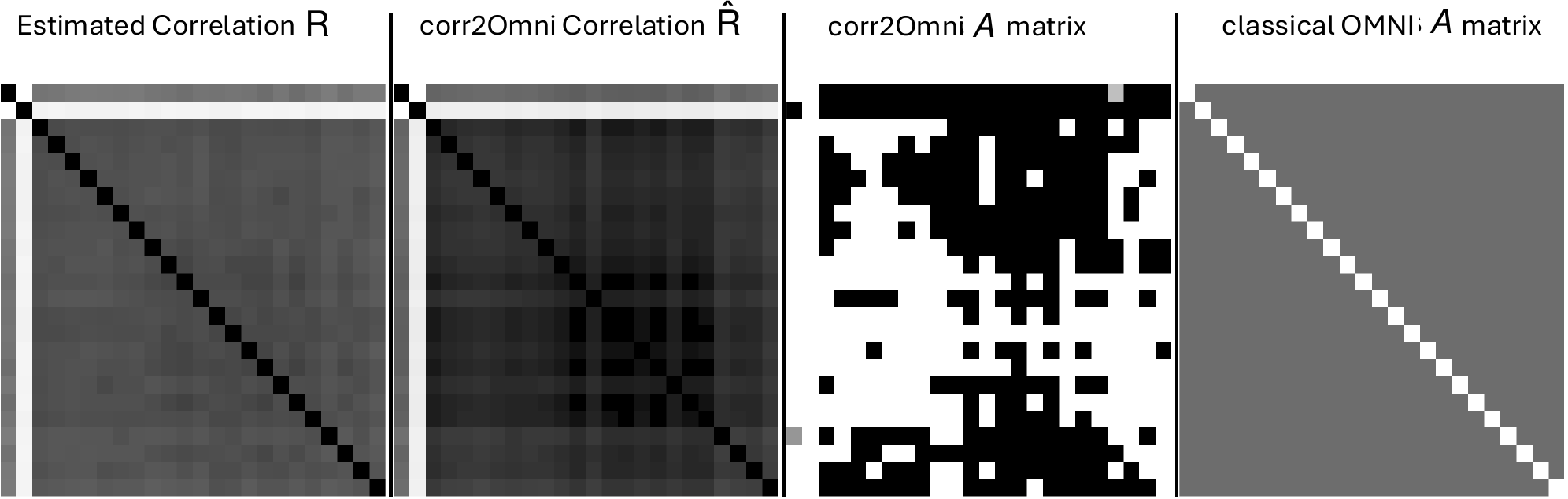}
  \caption{
From L to R: Detail of the estimated correlation matrix between the networks in the Aplysia motor program (L1) and the induced correlation provided by the \texttt{corr2Omni} estimate (L2).
Detail of the obtained $\mathcal{A}$ matrix obtained by \texttt{corr2Omni} (L3) and classical OMNI (L4) in the Aplysia experiment.
Note that in these $\mathcal{A}$ matrices, the diagonal has been 0'd out to allow for the off-diagonal detail to be observed.
Large values ($\leq1$) are darker grey and smaller colors ($\geq 0$) are lighter gray/white.
}
  \label{fig:c2o_aplysia1}
\end{figure}

In our next experiment, we consider a optical electrophysiological recording of neurons in the dorsal pedal ganglia of the \emph{Aplysia californica} gastropod from \cite{Aplysia}.
In the 20 minute recording of 82 neurons, a stimulus is applied at minute 1 to the dorsal pedal ganglion to elicit the Aplysia's escape motor program.
This data was analyzed in our motivating work \cite{pantazis2022importance}, where a carefully curated $\mathbf{C}$ tensor was constructed that allowed for finer grain analysis than that provided by classical OMNI.
Our goal in the present experiment is to automate the choice of $\mathbf{C}$ using our \texttt{corr2Omni} procedure.
We begin our analysis by converting this 20 minute recording into a time series of 24 correlation matrices (each accounting for a 50 second window so that the stimulus is not applied at the edge of two windows).
The subsequent 24 correlation networks are created by using the \texttt{sttc} algorithm of \cite{cutts} to convert the 50 second spike trains into correlation matrices (and with negative and non-computable correlations set to 0). 

Once the time-series of graphs is obtained, we embed the time series using two methods: classical OMNI and \texttt{corr2Omni}.
As in Section \ref{sec:soc}, we first estimate the edge-wise correlation matrix for each pair of graphs via alignment strength.
We then use this target $\mathfrak{R}$ in \texttt{corr2Omni} to estimate the weight matrix 
$\mathcal{A}$ using this $\mathfrak{R}$ and $\mathbf{R}=\mathfrak{R}$; 
a plot of this $\mathfrak{R}$ versus the induced correlation $\hat{\mathfrak{R}}$ estimated by \texttt{corr2Omni} is shown in Figure \ref{fig:c2o_aplysia1}.
We again note the trend of the Omnibus framework (exhibited here in \texttt{corr2Omni}; though to a lesser extent here than classical OMNI) to amplify/increase the existing correlation.
The obtained $\mathcal{A}$ matrices from \texttt{corr2Omni} as well as that from classical OMNI are displayed in Figure \ref{fig:c2o_aplysia1}; note the heterogeneity as compared to classical OMNI.
The distinctive pattern in \texttt{corr2Omni}'s $\mathcal{A}$ matrix allows for the estimated induced correlation to be much closer to the estimated inherent correlation than allowed for by classical OMNI.

\begin{figure}[t!]
  \centering
\includegraphics[width=1\linewidth]{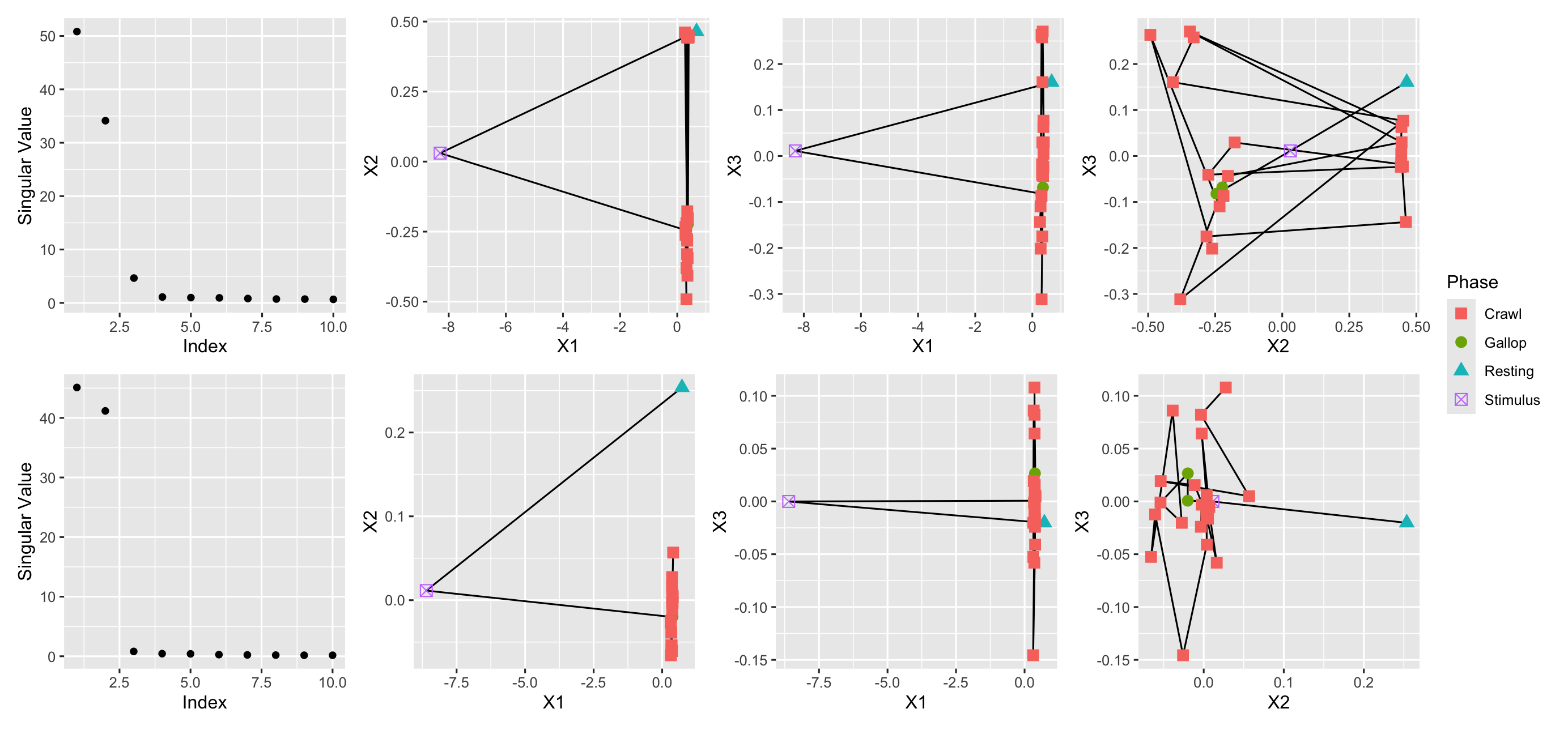}
  \caption{
Embedding the distance matrices for the classical OMNI embedding (top) and the \texttt{corr2Omni} embedding (bottom) using classical multidimensional scaling.
The first column shows the SCREE plot for each distance matrix.
Each row (columns 1--3) show the pairs plots for the 3 dimensional MDS embeddings of the distance matrices with $X1$, $X2$, and $X3$ corresponding to the first, second, and third MDS dimensions.
Lines in the pairs plots connect sequential graphs in the motor program.
}
  \label{fig:c2o_aplysia3}
\end{figure}

We embed the two Omnibus matrices (choosing the embedding dimension separately and automatedly according to the discussion in Section \ref{sec:omni}), and compute pairwise distances across the networks in the embedding space.
We then cluster these distance matrices using hierarchical clustering (with Ward's linkage; \texttt{ward.D2} in \texttt{R}). 
Dendrograms for the hierarchically clustered classical OMNI graphs (Top) and \texttt{corr2Omni} graphs (Bottom) are plotted in Figure \ref{fig:c2o_aplysia4}.  
Each dendrogram is colored to show (from L to R) the 2, 3,4, and 5 clusters obtained by the appropriate cut of the tree.  
Note the y-axis is on a square-root scale to allow the finer detail in the lower portion of the trees to be observed.
From these dendrograms, we can see (as noted in \cite{pantazis2022importance}) that the flat correlation structure induced by classical OMNI masks key signal in the motor program that is recovered by \texttt{corr2Omni}.
To wit, while both methods are able to isolate the stimulus (graph 2), classical OMNI does not distinguish the unique role of graph 1 as the lone control graph (before the stimulus).
We do note that if we cluster into 6 clusters, classical OMNI does isolate graph 1 as a singleton cluster as \texttt{corr2Omni} does from 3 clusters onward.
Moreover, if we look at the 4 and 5 cluster setting, \texttt{corr2Omni} is better able to recover the distinct phases of the motor program: resting/control state (in graph 1), then a stimulus (graph 2) followed by a gallop phase (graphs 3 and 4) and then a crawl phase.
The purple cluster is (roughly) capturing the gallop with the green and blue the subsequent crawling. 
Indeed, inferring the true labels of the motor program for Figure 4 in \cite{Aplysia} (roughly) as gallop lasting from graphs 3 to 4 and then the crawl, the ARI of classical OMNI with 4 clusters is 0.035 and of \texttt{corr2Omni} 0.195.

In Figure \ref{fig:c2o_aplysia3}, we plot the classical multidimensional scaling embedding of the distance matrices for the classical OMNI embedding (top) and the \texttt{corr2Omni} embedding (bottom).
The first column shows the SCREE plot for each distance matrix.
The figure reinforces the observation that \texttt{corr2Omni} does better than classical OMNI at isolating the unique resting state graph (here, the cyan triangle).
Moreover, we see that while both methods do a fairly good job of clustering the gallop phase (see the clustering of the green dots in the $X2$ versus $X3$ plot on the right), the transition from gallop to crawl in \texttt{corr2Omni} keeps the resting state distinct from the sequential crawl states.

\begin{figure}[t!]
  \centering
\includegraphics[width=1\linewidth]{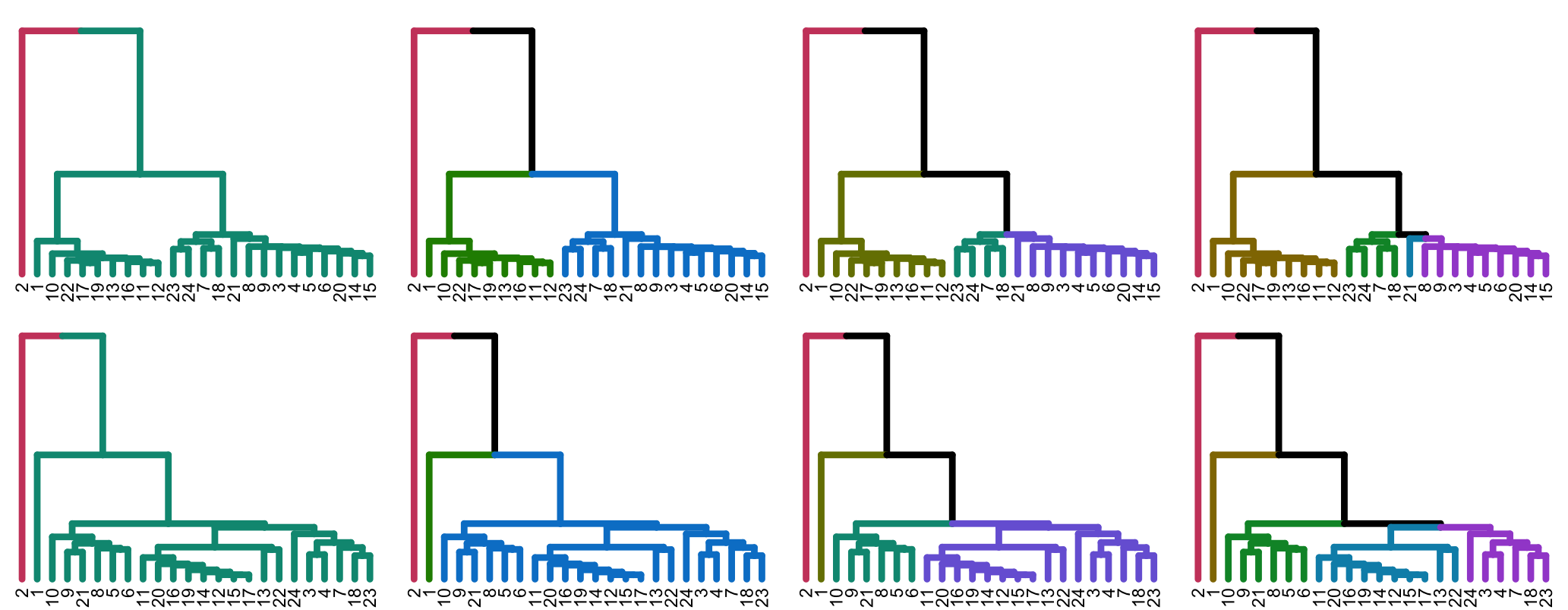}
  \caption{
Dendrograms for the hierarchically clustered classical OMNI graphs (Top) and \texttt{corr2Omni} graphs (Bottom).  
Each dendrogram is colored to show the (from left to right) 2, 3, 4, 5 clusters obtained by the appropriate cut of the tree.  Note the y-axis for the top and bottom rows have been separately scaled to allow the finer detail in the lower portion of the tree to be observed.
Note, among other improvements, that the \texttt{corr2Omni} dendrograms isolate the unique resting state (graph 1) whereas classical OMNI does not.
}
  \label{fig:c2o_aplysia4}
\end{figure}

\section{Conclusion and future work}
\label{sec:disc3}

Empirical evidence suggests that joint embedding algorithms induce correlation across the networks in the embedding space.
In the Omnibus embedding setting of \cite{pantazis2022importance}, this phenomena is best understood and the detrimental effect of sub-optimal Omnibus matrix constructions is demonstrated.
That said, the better constructions in \cite{pantazis2022importance} required careful curation and were not able to be automatized.
Taking this as our starting point, this present work addresses two key questions in the Omnibus embedding framework:
the \textit{correlation--to--OMNI} problem and \textit{flat correlation} problem. 
In the flat correlation problem, working in  a subspace of the fully general Omnibus matrices---namely the WOMNI or Weighted OMNI setting---we demonstrated a lower bound for the flat correlation and we showed that classical OMNI induces the maximum flat correlation among all WOMNI matrices. 
In the correlation--to--OMNI problem, we presented an algorithm---named \texttt{corr2Omni}---that approximates the optimal matrix of generalized Omnibus weights $\mathcal{A}$ from a given correlation matrix computed by the pairwise graph correlations.
In the WOMNI setting, converting from the weight matrix $\mathcal{A}$ to the Omnibus tensor $\mathbf{C}$ is immediate ($\mathbf{C}$ is needed for the embedding to be computed), though we are not yet aware of how to do this precisely in the more generalized Omnibus setting.
This is an open problem (as far as we are aware) and would be key for allowing for better, more general Omnibus constructions.

Nonetheless, even in the WOMNI restricted space, we show in simulations and a pair of real data networks the capacity of \texttt{corr2Omni} to induce less flat correlation and better structured correlation, both of which can measurably impact subsequent inference.
 and we demonstrated the accuracy of the algorithm on simulated data.
 Computationally optimizing the \texttt{corr2Omni} algorithm is another natural step, though the algorithm runs very quickly when there $m\approx 30$ networks.
Moreover, little is known about the performance of \texttt{corr2Omni} from a theoretical perspective, or about the objective function gap in general.

\vspace{3mm}

\noindent\textbf{Acknowledgements:} This material is based on research sponsored by NIH R01 NS121220 and by the Air Force Research
Laboratory (AFRL) and Defense Advanced Research Projects Agency (DARPA) under agreement number
FA8750-20-2-1001. 
The U.S. Government is authorized to reproduce and distribute reprints for Governmental purposes
notwithstanding any copyright notation thereon. The views and conclusions contained herein are
those of the authors and should not be interpreted as necessarily representing the official policies
or endorsements, either expressed or implied, of the AFRL and DARPA or
the U.S. Government.
The authors also gratefully acknowledge the support of the JHU HLTCOE.  We would also like to thank Dr. Evan Hill (at RFU) for help in procuring the Aplysia data.

\bibliographystyle{plain}
\bibliography{biblio,biblio_summary}
\appendix
\section{Proofs}
\noindent\normalfont
We provide proofs and explanations for the results of Section \ref{sec:flatcorr}. First, we prove Eq. (\ref{eq:obs2}) and then, we compute the lower bound from Subsection \ref{ssec:bounds}. Next, we present how the examples with smaller flat correlation arise, and finally, we prove Theorem \ref{thm:maxOMNI}.

\subsection{Proof of Eq. (\ref{eq:obs2})}
\label{obs:2}
\noindent
We first show and prove Lemma \ref{lem:constA}. 
\begin{lemma}
\label{lem:constA}
\noindent
\normalfont
In the weighted OMNI (WOMNI) setting, if all the cumulative weights are equal, then the only allowed value is $\frac{m+1}{2}$; that is, if $\alpha(q,q):=\sum\limits_{r=1}^m c^{(q)}_{q,r}=A$ for all $q\in[m]$, then $A=\frac{m+1}{2}$.
\end{lemma}
\begin{proof}
First note that in each row $q$ the cumulative weights must sum up to $m$, i.e., for each $q\in[m]$, $\sum_{\ell=1}^{m}\alpha(q,\ell)=m$. This implies that (under the assumption that $\alpha(q,q)=A$) 
$$\sum_{\ell\neq q}\alpha(q,\ell)=m-A.$$

\noindent
Now, observe that $\alpha(q,\ell)=c^{(\ell)}_{q,\ell}$ for any $\ell\neq q$ (This is because of the WOMNI architecture), hence $\sum\limits_{\ell\neq q}c^{(\ell)}_{q,\ell}=m-A.$ 
Next, we have the following $2m$ linear constraints: \begin{align}
    \sum\limits_{r\neq q}c^{(q)}_{q,r}&=A-1 \text{ for all }q\in[m];\label{eq:sumr}\\
    \sum\limits_{\ell\neq q}c^{(\ell)}_{q,\ell}&=m-A \text{ for all }q\in[m].\label{eq:suml}
\end{align}
Next, note that it always holds that (recall that $c^{(q)}_{a,b}=c^{(q)}_{b,a}$) \begin{equation}
    \label{eq:lemma1}
\sum\limits_q \sum\limits_{\ell\neq q}c^{(\ell)}_{q,\ell} =\sum\limits_q \sum\limits_{r\neq q}c^{(q)}_{q,r}.
\end{equation}
Summing Eqs. \ref{eq:suml} and \ref{eq:sumr} over $m$, then 
yields that $m^2-mA=mA-m\Rightarrow A=\frac{m+1}{2}$.
\end{proof}

\noindent
To visualize Eq. \ref{eq:lemma1}, notice that
\begin{align*}
    \sum\limits_q \sum\limits_{\ell\neq q}c^{(\ell)}_{q,\ell}&=\vec{1}_{m}^{\,T}
    \begin{bmatrix}
        c^{(1)}_{1,2}& c^{(1)}_{1,3} &\cdots &c^{(1)}_{1,m}\\ c^{(2)}_{2,1}& c^{(2)}_{2,3} &\cdots &c^{(2)}_{2,m}\\
        \vdots&\vdots&\ddots&\vdots\\
        c^{(m)}_{m,1}& c^{(m)}_{m,2} &\cdots &c^{(m)}_{m,m-1}
    \end{bmatrix}\vec{1}_{m-1}\\
    &=\vec{1}_{m}^{\,T}
    \begin{bmatrix}
        0&c^{(1)}_{1,2}& c^{(1)}_{1,3} &c^{(1)}_{1,4}&\cdots &c^{(1)}_{1,m}\\ c^{(2)}_{2,1}& 0& c^{(2)}_{2,3} &c^{(2)}_{2,4}&\cdots &c^{(2)}_{2,m}\\
        c^{(3)}_{3,1}& c^{(3)}_{3,2}& 0& c^{(3)}_{3,4} &\cdots &c^{(3)}_{3,m}\\
        \vdots&\vdots&\vdots&\vdots&\ddots&\vdots\\
        c^{(m)}_{m,1}& c^{(m)}_{m,2} &c^{(m)}_{m,3}&c^{(m)}_{m,4}&\cdots &0
    \end{bmatrix}\vec{1}_{m}\\
        &=\vec{1}_{m}^{\,T}
    \begin{bmatrix}
        0&c^{(2)}_{2,1}& c^{(3)}_{3,1} &c^{(4)}_{4,1}&\cdots &c^{(m)}_{m,1}\\ 
        c^{(1)}_{1,2}& 0& c^{(3)}_{3,2} &c^{(4)}_{4,2}&\cdots &c^{(m)}_{m,2}\\
        c^{(1)}_{1,3}& c^{(2)}_{2,3}& 0& c^{(4)}_{4,3} &\cdots &c^{(m)}_{m,3}\\
        \vdots&\vdots&\vdots&\vdots&\ddots&\vdots\\
        c^{(1)}_{1,m}& c^{(2)}_{2,m} &c^{(3)}_{3,m}&c^{(4)}_{4,m}&\cdots &0
    \end{bmatrix}\vec{1}_{m}=\sum\limits_q \sum\limits_{r\neq q}c^{(q)}_{q,r}
\end{align*}

\noindent
\textit{Proof of Eq. (\ref{eq:obs2}):}
Eq. (\ref{eq:obs2}) derives from the row-sum constraints (i.e., $\sum_{q=1}^m \alpha(k,q)=m$ for any row $k\in[m]$) and Eq. (\ref{eq:lemma1}).
To wit (recalling in the WOMNI setting that $c^{(q)}_{i,j}=0$ if $i,j\neq q$)
\begin{align*}
    m^2&=\sum_k\sum_q\alpha(k,q)=\sum_q\alpha(q,q)+\sum_q\sum_{k\neq q}\alpha(k,q)=\sum_q\alpha(q,q)+\sum_q\sum_{k\neq q}\sum_\ell c^{(q)}_{k,\ell}\\
    &=\sum_q\alpha(q,q)+\sum_q\underbrace{\sum_{k\neq q}\sum_{\ell\neq q} c^{(q)}_{k,\ell}}_{=0}+\sum_q\sum_{k\neq q} c^{(q)}_{k,q}=\sum_q\alpha(q,q)+\sum_q\underbrace{\sum_{k\neq q} c^{(q)}_{q,k}}_{=\alpha(q,q)-1}
\end{align*}
This yields that $m^2=2\sum_q\alpha(q,q)-m$, and solving for $\sum_q\alpha(q,q)$ yields the desired result.

\subsubsection{Proof of Lemma 
\ref{lem:lowerbound} and Eq. \ref{eq:upperbound}}
\label{ssec:lowerbound}
We first provide a proof of the lower bound.\\
\textit{Proof of Lemma \ref{lem:lowerbound}}:
First, consider the special case where $m=2$.
In this case, the value of $\mathfrak{r}_{1,2}=\frac{3}{4}+\frac{\rho}{4}$ for any choice of off-diagonal weights.  This is explainable via the fact that $\rho(X,Y)=\rho(cX,Y/c)$ for any $c\neq 0$, and the off-diagonal weights playing the role of weighting the contributions of the two graphs.

Next, we consider the case where $m\geq 3$.
Let $q^*$ be an arbitrary element of $\text{argmax}_q\alpha(q,q)$, and let $q_*$ be an arbitrary element of $\text{argmin}_q\alpha(q,q)$ that is distinct from $q^*$.
Flat correlation implies that if $s\neq q^*,q_*$,
letting $D:=\alpha(s,s)$,  $\alpha^*=\alpha(q^*,q^*)$, and $\alpha_*=\alpha(q_*,q_*)$ we have
\begin{align}
    \label{eq:covtermLB}
    (\alpha^*-c^{(q^{*})}_{q^{*},s})^2&+(D-c^{(s)}_{q^{*},s})^2+\sum_{\ell\neq q^{*},s}(c^{(q^{*})}_{q^{*},\ell}-c^{(s)}_{s,\ell})^2=\notag\\&(\alpha_*-c^{(q_{*})}_{q_{*},s})^2+(D-c^{(s)}_{q_{*},s})^2+\sum_{\ell\neq q_{*},s}(c^{(q_{*})}_{q_{*},\ell}-c^{(s)}_{s,\ell})^2,
\end{align}
where $c^{(q^{*})}_{q^{*},s}=1-c^{(s)}_{q^{*},s}$ and $c^{(q_{*})}_{q_{*},s}=1-c^{(s)}_{q_{*},s}$. 

In order to find a lower bound of $\alpha^*-\alpha_*$, we will proceed by minimizing the left hand side and maximizing the right hand size of Eq. (\ref{eq:covtermLB}), and then constructing the bound suitably. 
For the left hand side, note that
\begin{align*}
 (\alpha^*-c^{(q^{*})}_{q^{*},s})^2+(D-c^{(s)}_{q^{*},s})^2&+\sum_{\ell\neq q^{*},s}(c^{(q^{*})}_{q^{*},\ell}-c^{(s)}_{s,\ell})^2\geq (\alpha^*-c^{(q^{*})}_{q^{*},s})^2+(D-(1-c^{(q^*)}_{q^{*},s}))^2
 \end{align*}
 As a function of $c^{(q^{*})}_{q^{*},s}$, this is minimized by
 $$c^{(q^{*})}_{q^{*},s}=\frac{1}{2}(1+\alpha^*-D)$$
 If $\alpha^*-D> 1$, then the minimizer is outside the $[0,1]$ interval, and the minimum over $c^{(q^{*})}_{q^{*},s}\in[0,1]$ is achieved when $c^{(q^{*})}_{q^{*},s}=1$ yielding a lower bound of 
 $$(\alpha^*-1)^2+D^2$$
 If $\alpha^*-D\leq 1$, then the minimizer is inside the $[0,1]$ interval, and the minimum is achieved at 
 \begin{align*}
 (\alpha^*-c^{(q^{*})}_{q^{*},s})^2+(D-(1-c^{(q^*)}_{q^{*},s}))^2&\geq \left(\alpha^*-\frac{1}{2}(1+\alpha^*-D)\right)^2+
 \left(D-\frac{1}{2}(1-\alpha^*+D)\right)^2\\
 &=\frac{1}{2}(\alpha^*-1+D)^2\\
 &\geq \frac{1}{4}(2\alpha^*-2)^2+\frac{1}{4}(2D-1)^2\\
 &=(\alpha^*-1)^2+(D-1/2)^2
\end{align*}
We will use this lower bound as the global lower bound moving forward.
Similarly, for the right hand side we have
\begin{align*}
    (\alpha_*&-c^{(q_{*})}_{q_{*},s})^2+(D-c^{(s)}_{q_{*},s})^2+\sum_{q\neq q_{*},s}(c^{(q_{*})}_{q_{*},q}-c^{(s)}_{s,q})^2\\
    &\leq \alpha_*^2-2\alpha_*c^{(q_{*})}_{q_{*},s}+(c^{(q_{*})}_{q_{*},s})^2+D^2-2D(1-c^{(q_{*})}_{q_{*},s})+(1-c^{(q_{*})}_{q_{*},s})^2+(m-2)\\
    &\leq \alpha_*^2-2\alpha_*c^{(q_{*})}_{q_{*},s}+(c^{(q_{*})}_{q_{*},s})^2+D^2-2\alpha_*(1-c^{(q_{*})}_{q_{*},s})+(1-c^{(q_{*})}_{q_{*},s})^2+(m-2)\\
    &\leq \alpha_*^2-2\alpha_*+1+D^2+(m-2)\\
    &\leq  (\alpha_*-1)^2+D^2+(m-2).
\end{align*}
Combining bounds, we have that
\begin{align}
    \label{eq:LB1}
    &(\alpha^*-1)^2-(\alpha_*-1)^2+(D-1/2)^2-D^2\leq m-2\notag\\
    &\Rightarrow(\alpha^*-\alpha_*)(\alpha^*+\alpha_*-2)\leq m+D-7/4\leq 2m-7/4\notag\\
    &\Rightarrow \alpha^*-\alpha_*\leq\frac{2m-7/4}{\alpha^*+\alpha_*-2}\notag\\&\xRightarrow [\alpha_*\geq 1]{\alpha^*\geq \frac{m+1}{2}} \alpha^*-\alpha_*\leq \frac{2m-7/4}{\frac{m+1}{2}-1}=4\underbrace{\left(\frac{m-7/8}{m-1}\right)}_{\leq 1.0625\text{ for }m\geq 3}\leq 4.5.
\end{align}
From Eq. (\ref{eq:obs2}), we have that 
\begin{align*}
\alpha_* \leq \frac{m+1}{2}\Rightarrow \alpha^*-4.5\leq \frac{m+1}{2}\Rightarrow \alpha^*\leq \frac{m+10}{2}.
\end{align*}
and 
\begin{align*}
\alpha^* \geq \frac{m+1}{2}\Rightarrow \alpha_*+4.5\geq \frac{m+1}{2}\Rightarrow \alpha_*\geq \frac{m-8}{2}.
\end{align*}
Plugging this bound into Eq. (\ref{eq:naivebound}) we obtain the desired result of Lemma \ref{lem:lowerbound}. $\qed$

We next provide a proof of the upper bound in Eq. \ref{eq:upperbound}.
For the upper bound on $\mathfrak{r}_{s_1,s_2}$, note that
\begin{align*}
    \mathfrak{r}_{s_1,s_2}&\leq 1-(1-\rho)\frac{2(\alpha_*-1)^2 }{2m^2}\\
    &\leq 1-(1-\rho)\frac{2\alpha_*^2-4\alpha_*+2 }{2m^2}\\
    &\leq 1-(1-\rho)\frac{\frac{m^2}{2}-8m+32
    -
    2(m+1)+2 }{2m^2}\\
    &=\frac{3}{4}+\frac{\rho}{4}+(1-\rho)\left(\frac{5}{m}-\frac{16}{m^2}\right)
\end{align*}
as desired $\qed$


\subsection{Examples with smaller flat correlation}
\label{eg:example}
\subsubsection{m=3 case}
In the $m=3$ case, we seek to answer the following question:  How do we derive $\mathfrak{M}_{3-}$ and  $\mathfrak{M}_{3+}$, and do they induce the minimum flat correlation? 
To answer this question, let  
$x_1=c_{1,2}^{(1)}$, 
$x_2=c_{1,3}^{(1)}$, 
$y_1=c_{1,2}^{(2)}$, 
$y_2=c_{2,3}^{(2)}$, 
$z_1=c_{1,3}^{(3)}$ and 
$z_2=c_{2,3}^{(3)}$. Then, the quadratic equations from Eq. (\ref{eq:WOMNI}) become (by setting $2\frac{1-\mathfrak{r}_{s_1,s_2}}{1-\rho}=c\big)$
\begin{align}
\label{eq:m=3}
    &x_2^2+y_2^2-x_2y_2+x_2+y_2+1=m^2c\nonumber\\
    &x_1^2+z_2^2-x_1z_2+x_1+z_2+1=m^2c\nonumber\\
    &y_1^2+z_1^2-y_1z_1+y_1+z_1+1=m^2c
\end{align}
and the linear constraints 
\begin{align*}
    x_1+y_1=1\\x_2+z_1=1\\y_2+z_2=1
\end{align*}
Forcing the $[0,1]$ restrictions, an integer solution that maximizes $m^2C$ is 
$$(x_1,x_2,y_1,y_2,z_1,z_2)=(1,0,0,1,1,0).$$
This solution corresponds to $\mathfrak{M}_{3+}$.
An additional integer solution is given by swapping $0$s to $1$s and vice versa. The flip-side solution is $(0,1,1,0,0,1)$ which corresponds to $\mathfrak{M}_{3-}$. For these solutions the covariance coefficients are equal to $m^2c=6$, hence $c=\frac{2}{3}$ and $\mathfrak{r}_{s_1,s_2}=\frac{2}{3}+\frac{\rho}{3}$. 
Is this $\mathfrak{r}_{s_1,s_2}$ the minimum possible value? The answer is \textbf{yes}. 
The linear constraints incorporated into Eq. \ref{eq:m=3} yield
\begin{align}
\label{eq:m=32}
    &x_2^2+y_2^2-x_2y_2+x_2+y_2+1=m^2c\nonumber\\
    &x_1^2+y_2^2+x_1y_2-3y_2+3=m^2c\nonumber\\
    &x_1^2+x_2^2-x_1x_2-2x_1-2x_2+4=m^2c
\end{align}
Equating these three quadratic equations, we obtain the feasible solution set $$(x_1,x_2,y_1,y_2,z_1,z_2)=(x_1,1-x_1,1-x_1,x_1,x_1,1-x_1).$$
Substituting back to Eq. (\ref{eq:m=32}) we have that $x_1(x_1-1)=\frac{m^2c}{3}-1$. 
Since $x_1\in[0,1]$, the maximum value for the covariance coefficients is obtained when either $x_1=0$ or $x_1=1$ which result in the $\mathfrak{M}_{3-}$ and  $\mathfrak{M}_{3+}$ matrices.

\subsubsection{m=4 case}
\vspace{.2cm}
\noindent
As in the previous example, let 
$x_1=c_{1,2}^{(1)},$
 $x_2=c_{1,3}^{(1)},$
 $x_3=c_{1,4}^{(1)},$ 
 $y_1=c_{1,2}^{(2)},$
 $y_2=c_{2,3}^{(2)},$
 $y_3=c_{2,4}^{(2)},$
 $z_1=c_{1,3}^{(3)},$
 $z_2=c_{2,3}^{(3)},$
 $z_3=c_{3,4}^{(3)}$ and 
 $w_1=c_{1,4}^{(4)},$ 
 $w_2=c_{2,4}^{(4)},$
 $w_3=c_{3,4}^{(4)}$. Then, the 
 $\binom{4}{2}$ covariance coefficients from Eq. (\ref{eq:WOMNI}) can be written as
\begin{align*}
    &(x_2^2+y_2^2-x_2y_2+x_2+y_2)+ (x_3^2+y_3^2-x_3y_3+x_3+y_3)+(x_2x_3+y_2y_3+1)=m^2c\\
    &(x_1^2+z_2^2-x_1z_2+x_1+z_2)+(x_3^2+z_3^2-x_3z_3+x_3+z_3)+(x_1x_3+z_2z_3+1)=m^2c\\
    &(x_1^2+w_2^2-x_1w_2+x_1+w_2)+(x_2^2+w_3^2-x_2w_3+x_2+w_3)+(x_1x_2+w_2w_3+1)=m^2c\\
    &(y_1^2+z_1^2-y_1z_1+y_1+z_1)+(y_3^2+z_3^2-y_3z_3+y_3+z_3)+(y_1y_3+z_1z_3+1)=m^2c\\
    &(y_1^2+w_1^2-y_1w_1+y_1+w_1)+(y_2^2+w_3^2-y_2w_3+y_2+w_3)+(y_1y_2+w_1w_3+1)=m^2c\\
    &(z_1^2+w_1^2-z_1w_1+z_1+w_1)+(z_2^2+w_2^2-z_2w_2+z_2+w_2)+(z_1z_2+w_1w_2+1)=m^2c
\end{align*} along with the associated linear constraints 
\begin{align*}
    x_1+y_1=1\\x_2+z_1=1\\x_3+w_1=1\\y_2+z_2=1\\y_3+w_2=1\\z_3+w_3=1
\end{align*}

\noindent
Now, let $(x_1,x_2,y_1,y_2,z_1,z_2)=(1,0,0,1,1,0).$ Then, the system reduces to
\begin{align}
\label{eq:m=4}
    &x_3^2+y_3^2-x_3y_3+2y_3+x_3=m^2c-3\\
    &x_3^2+z_3^2-x_3z_3+2x_3+z_3=m^2c-3\nonumber\\
    &y_3^2+z_3^2-y_3z_3+2z_3+y_3=m^2c-3\nonumber\\
    &w_2^2+w_3^2+w_2w_3+w_3=m^2c-3\nonumber\\
    &w_1^2+w_3^2+w_1w_3+w_1=m^2c-3\nonumber\\
    &w_1^2+w_2^2+w_1w_2+w_2=m^2c-3\nonumber
\end{align}
\noindent
Using the linear constraints and equating the following pairs of equations: $(1,6)$,  $(2,5)$ and $(3,4)$, we get the hyperbolas
\begin{align*}
    &w_1w_2+2w_2+w_1=2\\&w_1w_3+2w_1+w_3=2\\&w_2w_3+2w_3+w_2=2.
\end{align*}
Equating the hyperbolas together, we get the solution set $$(w_1,w_2,w_3)=(x,x,x), \text{ for } 0\leq x \leq 1.$$ 
\noindent
Substituting back to the hyperbolas we obtain $w_1=w_2=w_3=\frac{\sqrt{17}-3}{2}$. Hence, $x_3=y_3=z_3=\frac{5-\sqrt{17}}{2}$. Further, $m^2c=2(21-4\sqrt{17})\approx 9.01515$ (Note that the classical OMNI case corresponds to $m^2c=8$). 
The induced flat correlation is then
$$\mathfrak{r}_{i,j}=\frac{4\sqrt{17}-5}{16}+\frac{21-4\sqrt{17}}{16}\rho\approx 0.72+0.28\rho,$$
whereas in the classical case it is
$$\mathfrak{r}_{i,j}=\frac{3}{4}+\frac{\rho}{4},$$
which is larger when $\rho<1$.

\subsection{Proof of Theorem \ref{thm:maxOMNI} (Lagrange multipliers)}
\label{sec:lagrange}
\noindent
Below, we demonstrate that by solving two Lagrange systems sequentially, we get that classical OMNI maximizes the flat correlation problem. The first Lagrange system has the covariance coefficient terms as Lagrange equations and the cumulative weight equations as Lagrange constraints. The second Lagrange system has the same equations as  the first system but written in terms of the (updated) cumulative weights and the only constraint is given by Eq. (\ref{eq:obs2}).


The worst-case flat correlation (i.e., the largest gap between $\rho=\rho_{s_1,s_2}$ and $\mathfrak{r}_{s_1,s_2}$) is achieved in the WOMNI setting when the common value of 
\begin{align*}
\sum_q \beta_{q,s_1,s_2}&=\Big( \alpha(s_1,s_1) - c^{(s_1)}_{s_2,s_1} \Big)^2 + \Big( \alpha(s_2,s_2) - c^{(s_2)}_{s_1,s_2} \Big)^2 + \sum_{q\neq s_1,s_2}\Big(c_{s_1,q}^{(q)} - c_{s_2,q}^{(q)} \Big)^2.\\
&=\left(1+\sum_{q\neq s_1,s_2} c_{s_1,q}^{(s_1)}\right)^2+
\left(1+\sum_{q\neq s_1,s_2} c_{s_2,q}^{(s_2)}\right)^2+
\sum_{q\neq s_1,s_2}(c_{s_1,q}^{(s_1)}-c_{s_2,q}^{(s_2)})^2\\
&=
\left(1+\sum_{q\neq s_1,s_2} c_{s_1,q}^{(s_1)}\right)^2+
\left(1+\sum_{q\neq s_1,s_2} c_{s_2,q}^{(s_2)}\right)^2+
\sum_{q\neq s_1,s_2}\left((c_{s_1,q}^{(s_1)})^2+(c_{s_2,q}^{(s_2)})^2-2c_{s_1,q}^{(s_1)}c_{s_2,q}^{(s_2)}\right)\\
&=
1+\sum_{q\neq s_1,s_2} 2c_{s_1,q}^{(s_1)}+\sum_{q\neq s_1,s_2} 2(c_{s_1,q}^{(s_1)})^2+\sum_{q\neq s_1,s_2}\sum_{\ell\neq s_1,s_2,q}c_{s_1,q}^{(s_1)} c_{s_1,\ell}^{(s_1)}\\
&\hspace{5mm}+ 
1+\sum_{q\neq s_1,s_2} 2c_{s_2,q}^{(s_2)}+\sum_{q\neq s_1,s_2} 2(c_{s_2,q}^{(s_2)})^2+\sum_{q\neq s_1,s_2}\sum_{\ell\neq s_1,s_2,q}c_{s_2,q}^{(s_2)} c_{s_2,\ell}^{(s_2)}\\
&\hspace{5mm}- \sum_{q\neq s_1,s_2} 2(1-c_{q,s_1}^{(q)}-c_{q,s_2}^{(q)}+c_{q,s_1}^{(q)}c_{q,s_2}^{(q)})\\
&=
1+\sum_{q\neq s_1,s_2} 2c_{s_1,q}^{(s_1)}+\sum_{q\neq s_1,s_2} 2(c_{s_1,q}^{(s_1)})^2+\sum_{q\neq s_1,s_2}\sum_{\ell\neq s_1,s_2,q}c_{s_1,q}^{(s_1)} c_{s_1,\ell}^{(s_1)}\\
&\hspace{5mm}+ 
1+\sum_{q\neq s_1,s_2} 2c_{s_2,q}^{(s_2)}+\sum_{q\neq s_1,s_2} 2(c_{s_2,q}^{(s_2)})^2+\sum_{q\neq s_1,s_2}\sum_{\ell\neq s_1,s_2,q}c_{s_2,q}^{(s_2)} c_{s_2,\ell}^{(s_2)}\\
&\hspace{5mm}+ \sum_{q\neq s_1,s_2} 2c_{q,s_1}^{(q)}+2c_{q,s_2}^{(q)}-2c_{q,s_1}^{(q)}c_{q,s_2}^{(q)}-2
\end{align*}
is minimized.
Rather than minimize this common value correctly, we will seek instead the weights that minimize the sum 
$$
f(\vec{c})=\sum_{\{s_1,s_2\}\in\binom{[m]}{2}} \sum_q \beta_{q,s_1,s_2}.$$
We will show that these minimizing weights induce flat correlation and hence are the solution to our optimization problem.
Collecting the terms in $f(\vec{c})$, we see that for any $i,h,\ell$, the coefficient 
\begin{align*}
&\text{of } c_{i,h}^{(i)}\,\,\text{ is } \underbrace{2(m-2)}_{\substack{
\text{ from the }(m-2)\,\,\sum_q\beta_{q,i,j}\text{ terms}\\
\text{where }j\neq i,h}}+
\underbrace{2(m-2)}_{\substack{
\text{ from the }(m-2)\,\,\sum_q\beta_{q,h,j}\text{ terms}\\
\text{where }j\neq i,h}};\\
&\text{of } (c_{i,h}^{(i)})^2\,\,\text{ is } \underbrace{2(m-2)}_{\substack{
\text{ from the }(m-2)\,\,\sum_q\beta_{q,i,j}\text{ terms}\\
\text{where }j\neq i,h}};\\
&\text{of } c_{i,h}^{(i)}c_{i,\ell}^{(i)}\,\,\text{ is } \underbrace{2(m-3)}_{\substack{
\text{ from the }(m-3)\,\,\sum_q\beta_{q,i,j}\text{ terms}\\
\text{where }j\neq i,h.\ell}}-
\underbrace{2}_{
\text{ from the }\sum_q\beta_{q,h,\ell}\text{ term}}
\end{align*}
Our first step will be to consider fixed total weights $a_i=\alpha(i,i)\geq 1$, and then to minimize $g(\vec{c})$ with these fixed weights acting as constraints on the feasible region (as each weight then gives the equation $1+\sum_{q\neq i}c^{(i)}_{i,q}=a_i$)
Note that we will not impose a positivity constraint on the vector of $c$'s (as this complicated the quadratic program), and we note the obtained minimizing solution will, thankfully, satisfy this condition nonetheless. 

The obtained quadratic program can be written in matrix form
\begin{align}
    \text{min }&\frac{1}{2}(\vec c)^T B \vec c+ q^T\vec c+r\label{eq:KKT1}\\
    \text{s.t. }&A\vec c=\vec b,\notag
\end{align}
where (noting all vectors are column vectors by default, and $\mathbb{J}_{m-1}$ is the $m-1\times m-1$ matrix of all $1$'s)
\begin{itemize}
    \item[i.] $\vec c=(c^{(1)}_{1,2},c^{(1)}_{1,3},\cdots,c^{(1)}_{1,m},c^{(2)}_{2,1},c^{(2)}_{2,3},\cdots,c^{(2)}_{2,m},\cdots,c^{(m)}_{m,1},c^{(m)}_{m,2},\cdots,c^{(m)}_{m,m-1})\in\mathbb{R}^{m(m-1)}$;  
    \item[ii.] $B$ is an $m(m-1)\times m(m-1)$ block matrix given by $B=I_m\otimes M$ and $M$ is an $(m-1)\times (m-1)$ matrix
$M=2(m-4)\bJm+2mI_{m-1}$;  this gives us that $B$ is invertible and positive definite, with inverse equal to 
$$B^{-1}=I_m\otimes \frac{1}{2m}\left(I_{m-1}-\frac{m-4}{(m-2)^2}\bJm\right)$$
    \item[iii.] $q=4(m-2)\mathds{1}_{m(m-1)}=4(m-2)\text{vec}(\mathbb{J}_{m-1,m})$;
    \item[iv.] $r=-\binom{m}{2}2(m-3)$;
    \item[v.] $A=I_m\otimes (\mathds{1}_{m-1})^T\in\mathbb{R}^{m\times m(m-1)}$, so that $A$ is rank $m$ and $\bA^TA=I_m\otimes \bJm\in\mathbb{R}^{m(m-1)\times m(m-1)}$;
    \item[vi.] $\vec b=\vec a-\mathds{1}_{m}$ where $a_i=\alpha(i,i)$.
\end{itemize}
For this quadratic program, we note that 
$$
B+\bA^TA=I_m\otimes \left(2m I_{m-1}+(2m-7)\bJm\right)
$$
has eigenvalues $2m+(m-1)(2m-7)$ (with multiplicity $m$) and $2m$ (with multiplicity $m(m-2)$), and hence $B+\bA^TA$ is positive definite and there is a unique solution to the above quadratic program.
The optimality conditions (see, for example, \cite{boyd2004convex}) here are then (where $\vec{\lambda}$ are the dual variables introduced in the KKT equations)
$$ \begin{pmatrix}
    B& \bA^T\\
    A& \mathbf{0}_{m,m}
\end{pmatrix} \begin{pmatrix}
    \vec{c}^{\,\,*}\\
    \vec{\lambda}^*
\end{pmatrix}=\begin{pmatrix}
    -q\\
    b
\end{pmatrix}$$
We first begin by analytically computing the inverse of the KKT coefficient matrix (see \cite{noble1977applied} for reference)
$$\begin{pmatrix}
    B& \bA^T\\
    A& \mathbf{0}_{m,m}
\end{pmatrix}^{-1}=\begin{pmatrix}
    B^{-1}-B^{-1}\bA^T(AB^{-1}\bA^T)^{-1}AB^{-1}& B^{-1}\bA^T(AB^{-1}\bA^T)^{-1}\\
    (AB^{-1}\bA^T)^{-1}AB^{-1}& -(AB^{-1}\bA^T)^{-1}
\end{pmatrix}$$
Some intermediary computations yield that
\begin{align*}
(AB^{-1}\bA^T)^{-1}&=\frac{2(m-2)^2}{(m-1)}I_m\\
AB^{-1}&=I_m\otimes \frac{1}{2(m-2)^2}(\mathds{1}_{m-1})^T
\end{align*}
so that
\begin{align*}
B^{-1}-B^{-1}&\bA^T(AB^{-1}\bA^T)^{-1}AB^{-1}\\
&=I_m\otimes \left(\frac{1}{2m}I_{m-1}-\frac{m-4}{2m(m-2)^2}\bJm-\frac{1}{2(m-1)(m-2)^2}\bJm\right)\\
&=I_m\otimes \frac{1}{2m}\left(I_{m-1}-\frac{1}{m-1}\bJm\right),
\end{align*}
and
\begin{align*}
B^{-1}\bA^T(AB^{-1}\bA^T)^{-1}=I_m\otimes \frac{1}{m-1}\mathds{1}_{m-1}.
\end{align*}
The optimizing point for the quadratic program is then uniquely defined via
\begin{align*}
    \vec{c}^{\,\,*}&=-\left(I_m\otimes \frac{1}{2m}\left(I_{m-1}-\frac{1}{m-1}\bJm\right)\right)q+\left(I_m\otimes \frac{1}{m-1}\mathds{1}_{m-1}\right)b\\
    &=-\frac{4(m-2)}{2m}\text{vec}\left(\bJmm-\frac{1}{m-1}\bJm\bJmm\right)+\left(I_m\otimes \frac{1}{m-1}\mathds{1}_{m-1}\right)b\\
    &=\frac{1}{m-1}\left((\vec{a}-1)\otimes \mathds{1}_{m-1}\right),
\end{align*} 
which is (luckily) nonnegative.

These $c$'s yields the objective function value of 
\begin{align*}
f(\vec{a})&=\sum_{\{i,j\}\in\binom{[m]}{2}}\left(
\left(1+\frac{m-2}{m-1}(a_i-1)\right)^2+
\left(1+\frac{m-2}{m-1}(a_j-1)\right)^2+
(m-2)\left(\frac{a_i-a_j}{m-1}\right)^2\right)\\
&=\sum_{\{i,j\}\in\binom{[m]}{2}}\left(
\left(\frac{m-2}{m-1}a_i+\frac{1}{m-1}\right)^2+
\left(\frac{m-2}{m-1}a_j+\frac{1}{m-1}\right)^2+
(m-2)\left(\frac{a_i-a_j}{m-1}\right)^2\right)\\
&=\frac{1}{(m-1)^2}\sum_{\{i,j\}\in\binom{[m]}{2}}\left(
\left((m-2)a_i+1\right)^2+
\left((m-2)a_j+1\right)^2+
(m-2)\left(a_i-a_j\right)^2\right)\\
&\propto 2\binom{m}{2}+(m-2)\sum_{\{i,j\}\in\binom{[m]}{2}}
\left(
(m-1)a_i^2+(m-1)a_j^2+2a_i+2a_j-2a_ia_j\right)
\end{align*}
We next seek the total weights (i.e., the vector of $a$'s) that minimizes this objective function.

The minimizer of this objective function (over $\vec{a}$ in any defined region) is the same as the minimizer of 
$$
f_2(\vec{a})=\frac{1}{2}(\vec{a})^T\underbrace{\begin{pmatrix}
    2(m-1)^2& -2 &\cdots&-2\\
    -2&2(m-1)^2 &\cdots&-2\\
    \vdots&\vdots&\ddots&\vdots\\
    -2&-2&\cdots&2(m-1)^2
\end{pmatrix}}_{=B_2}\vec{a}+2(m-1)(\mathds{1}_m)^T\vec{a}$$
where
\begin{align*}
    B_2=2\left((m^2-2m+2)I_m-\mathbb{J}_m\right)\\
    B_2^{-1}=\frac{1}{2(m^2-2m+3)}\left(I_m+\frac{1}{m^2-3m+3}  \mathbb{J}_m\right)
\end{align*}
Consider minimizing $f_2$ subject to the single constraint (given by Eq. \ref{eq:obs2}) that $\sum_i a_i=(\mathds{1}_m)^T\vec{a}=\frac{m(m+1)}{2}$.
There is a unique minimizing solution as
\begin{align*}
    B_2+\mathds{1}_m(\mathds{1}_m)^T=2(m^2-2m+2)I_m-\mathbb{J}_m
\end{align*}
has eigenvalues $2m^2-2m+2$ (with multiplicity $m-1$) and $2m^2-3m+2$ (with multiplicity $1$) and hence is positive definite.
Therefore, the constrained quadratic program has a unique solution, which can be found by solving (again with $\vec{\lambda}$ the KKT dual variables)
$$\begin{pmatrix}
    B_2& \mathds{1}_m\\
    (\mathds{1}_m)^T& 0
\end{pmatrix}\begin{pmatrix}
    \vec{a}^*\\
    \vec{\lambda}^*
\end{pmatrix}=\begin{pmatrix}
    2(m-1)\mathds{1}_m\\
    \frac{m(m+1)}{2}
\end{pmatrix}
$$
The optimal solution is then
\begin{align*}
    \vec{a}^*=-2(m-1)\left(B_2^{-1}-B_2^{-1}\mathds{1}_m(\mathds{1}_m^TB_2^{-1}\mathds{1}_m)^{-1}\mathds{1}_m^TB_2^{-1}\right)\mathds{1}_m+B_2^{-1}\mathds{1}_m(\mathds{1}_m^TB_2^{-1}\mathds{1}_m)^{-1}\frac{m(m+1)}{2}
\end{align*}
Now, 
\begin{align*}
    (\mathds{1}_m^TB_2^{-1}\mathds{1}_m)^{-1}=\frac{2(m^2-3m+3)}{m}\\
    B_2^{-1}\mathds{1}_m=\frac{1}{2(m^2-3m+3)}\mathds{1}_m
\end{align*}
so that
$$
B_2^{-1}\mathds{1}_m(\mathds{1}_m^TB_2^{-1}\mathds{1}_m)^{-1}\mathds{1}_m^TB_2^{-1}=\frac{1}{2m(m^2-3m+3)}\mathbb{J}_m
$$
and 
\begin{align*}
    \vec{a}^*&=-2(m-1)\left(B_2^{-1}-B_2^{-1}\mathds{1}_m(\mathds{1}_m^TB_2^{-1}\mathds{1}_m)^{-1}\mathds{1}_m^TB_2^{-1}\right)\mathds{1}_m+B_2^{-1}\mathds{1}_m(\mathds{1}_m^TB_2^{-1}\mathds{1}_m)^{-1}\frac{m(m+1)}{2}\\
    &=-\frac{2(m-1)}{2(m^2-3m+3)}\mathds{1}_m+\frac{2(m-1)}{2(m^2-3m+3)}\mathds{1}_m +\frac{m+1}{2}\mathds{1}_m\\
    &=\frac{m+1}{2}\mathds{1}_m
\end{align*}
and hence the optimal solution for inducing the maximal total correlation in the WOMNI setting is achieved by setting all $c^{(i)}_{i,j}=1/2$ for $i\neq j$ (i.e., by classic OMNI)
\endproof


\end{document}